\RequirePackage{fix-cm}
\documentclass[notitlepage]{article}

\usepackage{arxiv}
\usepackage[T1]{fontenc}
\usepackage[utf8]{inputenc}
\usepackage{microtype}
\usepackage{graphicx}
\usepackage{booktabs}
\usepackage{caption}
\usepackage{tabularx}
\usepackage{array}
\usepackage{enumitem}
\usepackage{tikz}
\usepackage{algorithm}
\usepackage{algorithmic}
\usepackage[numbers,square,sort&compress]{natbib}

\usepackage{amsmath}
\usepackage{amssymb}
\usepackage{mathtools}
\usepackage{amsthm}
\usepackage{xcolor}
\usepackage{needspace}

\AtBeginDocument{\normalsize\raggedbottom}

\theoremstyle{plain}
\newtheorem{theorem}{Theorem}[section]
\newtheorem{proposition}[theorem]{Proposition}
\newtheorem{lemma}[theorem]{Lemma}
\newtheorem{corollary}[theorem]{Corollary}
\newcommand{\restatedresultname}{}
\newtheorem*{restatedresult}{\restatedresultname}
\newenvironment{restatement}[3]{%
  \renewcommand{\restatedresultname}{#1~\ref{#2}}%
  \begin{restatedresult}[#3]%
}{\end{restatedresult}}

\DeclareMathOperator{\Var}{Var}
\newcommand{\E}{\mathbb E}
\newcommand{\Pp}{\mathbb P}
\newcommand{\BReg}{\operatorname{BReg}}
\newcommand{\cF}{\mathcal F}
\newcommand{\cC}{\mathcal C}

\usepackage{hyperref}
\usepackage{bookmark}

\hypersetup{
  colorlinks=true,
  citecolor=blue,
  linkcolor=blue,
  urlcolor=blue,
  pdftitle={Minimax-Optimality of Posterior Sampling for Reinforcement Learning},
  pdfauthor={Taewon Goo and Kihyuk Hong},
  pdfsubject={Posterior sampling for reinforcement learning},
  pdfkeywords={reinforcement learning, posterior sampling, Bayesian regret, minimax regret}
}

\let\appendixaddtocontents\addtocontents
\def\psrlcontentsfile{toc}
\renewcommand{\addtocontents}[2]{%
  \def\requestedcontentsfile{#1}%
  \ifx\requestedcontentsfile\psrlcontentsfile
  \else
    \appendixaddtocontents{#1}{#2}%
  \fi
}

\title{Minimax-Optimality of Posterior Sampling\\for Reinforcement Learning}
\author{Taewon Goo\\
  \normalfont KAIST\\
  \normalfont\texttt{tetrise9@kaist.ac.kr}
  \And
  Kihyuk Hong\\
  \normalfont KAIST\\
  \normalfont\texttt{kihyukh@kaist.ac.kr}}
\date{}

\begin{document}
\maketitle
\suppressfloats[t]

\begin{abstract}
Posterior sampling for reinforcement learning (PSRL) is one of the simplest and most effective exploration methods, but a basic question has remained open: does unmodified PSRL achieve minimax regret without structural assumptions on the prior? We answer yes. Exact vanilla PSRL is minimax optimal in leading-order Bayesian regret under arbitrary correlated priors. The difficulty is that a posterior-sampled transition model is coupled with its own continuation value. We overcome this with a common empirical transition reference that isolates the resulting value mismatch and a Bellman-based variance argument that controls it without an extra leading-order state-space factor. For finite-horizon, time-inhomogeneous tabular MDPs with unknown stochastic rewards, this yields the minimax \(\widetilde O(\sqrt{SAH^3K})\) regret rate under arbitrary joint priors over rewards and transitions. The same proof principle gives the minimax \(\widetilde O(d\sqrt{H^3K})\) rate for linear-mixture MDPs under arbitrary joint parameter priors.

\end{abstract}

\section{Introduction}
\label{sec:introduction}

\begin{table}[t]
\centering
\caption{Selected Bayesian regret guarantees for exact vanilla PSRL over $K$
episodes of horizon $H$, with $S$ states, $A$ actions, parameter dimension $d$,
and logarithmic factors suppressed. Comparisons use
the corresponding minimax rates: $\widetilde\Theta(\sqrt{SAH^2K})$ for
homogeneous tabular MDPs, $\widetilde\Theta(\sqrt{SAH^3K})$ for
time-inhomogeneous tabular MDPs, and
$\widetilde\Theta(d\sqrt{H^3K})$ for linear-mixture MDPs
\citep{azar2017,domingues2021,zhang2024,zhou2021}. For
\citep{moradipari2023}, $\lambda\le H+1$ is a prior-dependent value diameter
and $T_0$ a prior-dependent burn-in; their result also requires sufficiently
large $K$. Linear-mixture rows use the standard bounded-feature and
bounded-parameter assumptions.}
\label{tab:comparison}
\fontsize{8}{9.2}\selectfont
\setlength{\tabcolsep}{3pt}
\renewcommand{\arraystretch}{1.10}
\renewcommand{\tabularxcolumn}[1]{m{#1}}
\begin{tabularx}{\textwidth}{@{}
>{\raggedright\arraybackslash}m{1.75cm}
>{\centering\arraybackslash}m{1.30cm}
>{\raggedright\arraybackslash}m{2.70cm}
>{\raggedright\arraybackslash}m{2.40cm}
>{\centering\arraybackslash}m{3.80cm}
@{\hspace{3pt}}
>{\centering\arraybackslash}X@{}}
\toprule
Work & Time & Prior & Rewards & Bayesian regret & Comparison \\
\midrule
\multicolumn{6}{@{}l}{\emph{Tabular MDPs}}\\
\citet{osband2013}
& Homog. & Arbitrary & Unknown stochastic
& $\widetilde O(S\sqrt{AH^3K})$ & gap $\sqrt{SH}$ \\
\addlinespace[3pt]
\citet{osband2017}
& Homog. & Independent Dirichlet rows & Independent reward prior; sub-Gaussian noise
& $\widetilde O(\sqrt{SAH^3K})$ & gap $\sqrt H$ \\
\addlinespace[3pt]
\citet{moradipari2023}
& Inhomog. & Independent stages/components; consistency and occupancy assumptions
& Unknown; independent of transitions
& $\widetilde O(\lambda H\sqrt{SAK})+T_0$
& \shortstack{worst-case gap $\sqrt H$\\$T_0$ prior-dependent} \\
\addlinespace[3pt]
\textbf{This work}
& \textbf{Inhomog.} & \textbf{Arbitrary joint}
& \textbf{Unknown stochastic}
& $\boldsymbol{\widetilde O(\sqrt{SAH^3K})}$
& \textbf{minimax} \\
\midrule
\multicolumn{6}{@{}l}{\emph{Linear-mixture MDPs}}\\
\citet{li2024}
& Inhomog. & Independent stage parameters & Known deterministic
& $\widetilde O(d\sqrt{H^4K})$ & gap $\sqrt H$ \\
\addlinespace[3pt]
\textbf{This work}
& \textbf{Inhomog.} & \textbf{Arbitrary joint}
& \textbf{Known deterministic}
& $\boldsymbol{\widetilde O(d\sqrt{H^3K})}$
& \textbf{minimax} \\
\bottomrule
\end{tabularx}
\end{table}

Posterior sampling for reinforcement learning (PSRL) extends Thompson's
posterior-sampling principle \citep{thompson1933} to sequential decision
problems. Following the model-sampling approach of \citep{strens2000}, episodic
PSRL draws one MDP from the posterior at the start of each episode and follows
an optimal policy for that sample. This principle has proved versatile, with
extensions to continuous control \citep{fanming2021,flynn2026},
function approximation \citep{dann2021,li2024}, constrained reinforcement learning
\citep{agarwalcmdp2022}, Gaussian-process models \citep{bayrooti2025},
graph-structured reinforcement learning \citep{robert2025}, and preference-based
learning \citep{agnihotri2026}, among many other settings.

Despite this breadth, the theoretical understanding of vanilla PSRL remains
incomplete. Its Bayesian regret has been studied through frequentist-style
concentration arguments \citep{osband2013,osband2017} and information-theoretic
approaches \citep{lu2019,hao2022,moradipari2023}. The algorithm itself does not
require the prior to factor across states, actions, or stages, yet the sharpest
previous regret guarantees rely on such structure or additional posterior
regularity \citep{osband2017,moradipari2023,li2024}. Meanwhile, minimax-optimal
worst-case regret is already achievable by carefully designed frequentist
algorithms \citep{azar2017}, and minimax guarantees in episodic tabular MDPs
have also been established for modified posterior-based algorithms
\citep{tiapkin2022opsrl,tiapkin2022bayesucb}.

Even with minimax-optimal frequentist methods available, posterior sampling
remains an effective and actively developed approach to exploration, making the
statistical optimality of its canonical form a fundamental question. Indeed,
the minimax optimality of exact PSRL was explicitly left as a conjecture by
\citep{osband2017} and \citep{moradipari2023}, even under prior structures more
restrictive than those considered here. These developments leave a basic
question unresolved:

\Needspace{6\baselineskip}
\noindent\textbf{Open question.} \emph{Does vanilla PSRL achieve
minimax-optimal Bayesian regret under an arbitrary joint prior?}

We answer this question affirmatively, showing that the structural assumptions
used by sharper previous analyses are not intrinsic to vanilla PSRL: exact
posterior sampling attains the minimax leading regret rate even when the prior
jointly correlates all reward and transition components.

The main technical obstacle is model--value dependence: a posterior-sampled
transition model is coupled with its own continuation value, so sharp
fixed-direction concentration is not directly available, while value-uniform
control incurs an additional state-space factor.

\noindent\textbf{Technical novelty.}
The proof uses a common empirical reference: the empirical transition law
inserted between the sampled and true models. This decomposes the
transition error into a sampled-model term and a true-model term, each evaluated
in its own value direction, plus one term involving the difference of the two
value functions. A variance argument based on the Bellman recursion controls
this remaining mismatch through the same local errors already in the regret,
so the $\sqrt S$ cost of uniform control does not spread to the leading term.
The linear-mixture proof adapts the same variance idea to posterior-sampled
value functions.

Our contributions are:
\begin{itemize}[leftmargin=*]
    \item \emph{Minimax-optimal tabular PSRL.}
    Exact PSRL attains the minimax $K$-dependent regret order
    $\widetilde O(\sqrt{SAH^3K})$ for time-inhomogeneous tabular MDPs
    under arbitrary joint priors over rewards and transitions.
    \item \emph{Minimax-optimal linear-mixture PSRL.}
    We further prove the minimax $\widetilde O(d\sqrt{H^3K})$ regret rate
    for linear-mixture MDPs under arbitrary joint parameter priors.
    \item \emph{A new regret analysis.} We give a proof architecture that
    isolates the dependence between a sampled model and its value using a common
    empirical reference and a Bellman-based variance argument.
\end{itemize}

\subsection{Related Work}
\label{subsec:related-work}

\paragraph{\textbf{Vanilla PSRL analyses.}}
Earlier arbitrary-prior PSRL analysis remained polynomially suboptimal
\citep{osband2013}, while sharper results relied on independent Dirichlet or
component-wise priors and additional posterior regularity
\citep{osband2017,lu2019,moradipari2023}. In stochastic shortest-path models,
\citep{jafarnia2023} provide a close precedent for transferring fixed-model
confidence bounds to posterior samples, but their coordinate-wise analysis
retains a $\sqrt S$ minimax gap. \citep{moradipari2023} identify and correct
a proof error in the surrogate construction of \citep{hao2022}; we therefore
exclude the latter's claimed PSRL rate from Table~\ref{tab:comparison}.
Structured dependence through finite latent
mixtures has also been studied \citep{hong2022}, but does not cover an
arbitrary joint prior over the full MDP. Table~\ref{tab:comparison} summarizes
the closest established guarantees. A technical comparison of why these
existing proof routes do not close the arbitrary-prior minimax gap is given in
Section~\ref{subsec:proof-obstacles}.

\paragraph{\textbf{PSRL variants.}}
Several lines of work have modified PSRL to obtain provable regret guarantees,
using optimism \citep{agrawaljia2017,tiapkin2022opsrl}, posterior quantiles
\citep{tiapkin2022bayesucb}, sampled value predictors \citep{dann2021}, or
value-biased sampling \citep{zhang2022feelgood}. The tabular methods of
\citep{tiapkin2022opsrl,tiapkin2022bayesucb} attain minimax regret, but these
guarantees do not apply to vanilla PSRL, which uses a single unmodified posterior draw. An
analogous gap remained for PSRL in linear-mixture MDPs
\citep{osband2014,li2024}.

\section{Problem Setting and Vanilla PSRL}
\label{sec:problem-setting}

\subsection{Finite-Horizon Tabular MDPs}
\label{subsec:tabular-mdps}

Let $\mathcal S$ and $\mathcal A$ be finite sets with $|\mathcal S|=S$ and
$|\mathcal A|=A$. We consider a finite-horizon MDP
\(M=(\mathcal S,\mathcal A,H,s_1,\{J_h^M\}_{h=1}^H)\), where
$J_h^M(\cdot\mid s,a)$ is a distribution on $[0,1]\times\mathcal S$.
Conditional on $M$, each visit at stage $h$ produces a fresh joint draw
\[
(R_{k,h},S_{k,h+1})
\sim J_h^M(\cdot\mid S_{k,h},A_{k,h}).
\]
Every episode lasts for $H$ stages. The reward and transition rules may depend
on $h$, so the model is time-inhomogeneous; time-homogeneous MDPs are a special
case.

For clarity, the main text fixes a known initial state $s_1$.
Appendix~\ref{app:initial-state} shows that the same guarantees hold for any
model-independent initial-state distribution, known or unknown.

For every $(h,s,a)$, let $\nu_h^M(\cdot\mid s,a)$ and
$P_h^M(\cdot\mid s,a)$ be the reward and next-state distributions obtained
from $J_h^M$, and write
$r_h^M(s,a):=\mathbb E_M[R_h\mid S_h=s,A_h=a]$. For
$v\in\mathbb R^S$, use the shorthand
$P_h^Mv(s,a):=\sum_{s'}P_h^M(s'\mid s,a)v(s')$.
If reward and next state are conditionally independent within a visit, then
$J_h^M=\nu_h^M\otimes P_h^M$, giving the conventional tuple representation
$M=(\mathcal S,\mathcal A,H,s_1,\nu^M,P^M)$. Otherwise, the same two
marginals determine values and planning, while $J_h^M$ also describes the
joint observation used for posterior updating. Thus within-visit dependence
between reward and next state is allowed throughout.

For a deterministic stage-dependent Markov policy
$\pi=(\pi_1,\ldots,\pi_H)$, define
\nopagebreak[4]
\[
V_{h,M}^{\pi}(s):=\mathbb E_{M,\pi}\!\left[
\sum_{t=h}^H R_t\,\middle|\,S_h=s\right].
\]
Let $V_{h,M}^{\star}(s):=\max_{\pi}V_{h,M}^{\pi}(s)$, with
$V_{H+1,M}^{\pi}=V_{H+1,M}^{\star}=0$. Because rewards lie in $[0,1]$, all
value functions lie in $[0,H]$. We fix a deterministic tie-breaking rule and
write $\pi_M^\star$ for the resulting optimal policy in model $M$.

\subsection{Arbitrary Correlated Priors}
\label{subsec:arbitrary-priors}

Let $\rho$ be an arbitrary prior over the full collection
$\{J_h^M(\cdot\mid s,a)\}_{h,s,a}$, and draw $M^\star\sim\rho$. We allow
arbitrary dependence across stages, state--action rows, rewards, and
transitions, so observing one component may update beliefs about any other.
Conditional on a fixed model $M$, however, successive visits draw fresh
observations from its fixed row laws. We impose no conjugacy, factorization,
smoothness, posterior consistency, or guaranteed prior mass near the true
model; the measure-theoretic conventions are given in
Appendix~\ref{app:tabular-proof}.

\subsection{Vanilla Posterior Sampling}
\label{subsec:vanilla-psrl}

Let $\mathcal F_k$ contain the complete history before episode $k$, including
all previous observations and algorithmic randomness, and write
$\mathbb E_k[\cdot]:=\mathbb E[\cdot\mid\mathcal F_k]$.
Vanilla PSRL is summarized in Algorithm~\ref{alg:vanilla-psrl}. Each posterior
draw uses fresh algorithmic randomness, so conditional on $\mathcal F_k$,
$M_k$ is an independent draw from the posterior distribution of $M^\star$.
We study statistical regret assuming exact posterior sampling and exact
planning for the sampled finite MDP.

\begin{algorithm}[H]
\caption{PSRL}
\label{alg:vanilla-psrl}
\begin{algorithmic}[1]
\STATE \textbf{Input:} prior $\rho$
\FOR{episode $k=1,2,\ldots$}
    \STATE Sample $M_k\sim\mathbb P(M^\star\in\cdot\mid\mathcal F_k)$
    \STATE Compute an optimal policy $\pi_k=\pi_{M_k}^\star$
    \STATE Execute $\pi_k$ for one episode and update $\mathcal F_{k+1}$
\ENDFOR
\end{algorithmic}
\end{algorithm}

\Needspace{4\baselineskip}
Given a posterior draw, PSRL solves one known finite-horizon MDP per episode
\citep{osband2017}. Exact posterior inference may still be difficult for a
general prior.

\subsection{Bayesian Regret}
\label{subsec:bayesian-regret}

Define the regret in episode $k$ and the Bayesian regret over $K$ episodes by
\[
\begin{aligned}
\Delta_k&:=V_{1,M^\star}^{\star}(s_1)-V_{1,M^\star}^{\pi_k}(s_1),\\
\operatorname{BReg}_{\rho}(K)&:=\mathbb E\!\left[\sum_{k=1}^K\Delta_k\right].
\end{aligned}
\]
The expectation is over the true-model draw, posterior samples, and observed
trajectories. Conditional on $\mathcal F_k$, $M^\star$ and $M_k$ have the same
posterior distribution. Hence, for every bounded function $g$ of one model,
\[
\mathbb E_k[g(M^\star)]=\mathbb E_k[g(M_k)].
\]
This does not make the two models equal or make their components independent.
Section~\ref{subsec:posterior-cancellation} applies the identity to the optimal
value.

\section{Main Result and Proof Overview}
\label{sec:main-results}

We state the minimax-optimal tabular guarantee and then summarize why existing
PSRL analyses do not prove it and how our argument closes the gap.

\subsection{Minimax-Optimal Tabular Regret}
\label{subsec:tabular-main-result}

For the bounds below, let $L_K:=2+\log K$ and
$\beta_K:=\log(c_0S^2AHK^2)$, where $c_0$ is a sufficiently large universal
constant. Both are logarithmic in the problem parameters. Throughout the
tabular analysis, $C$ denotes a universal positive constant whose value may
change between displays.

\begin{theorem}[Minimax tabular regret under an arbitrary prior]
\label{thm:arbitrary-correlated-prior}
For every prior $\rho$ and every $K\ge 1$, exact vanilla PSRL satisfies
\[
\operatorname{BReg}_{\rho}(K)
\le
C\sqrt{SAH^3K\beta_KL_K}
+CS^2AH^2\beta_KL_K.
\]
\end{theorem}

The full statement, including the lower-order terms before absorption, is given
in Appendix~\ref{app:tabular-proof}.

For time-inhomogeneous tabular MDPs, the minimax lower bound is
$\Omega(\sqrt{SAH^3K})$ \citep{domingues2021}. The frequentist minimax
lower-bound construction can be transferred to Bayesian regret by viewing its
random hard instance as a prior. Hence the same rate lower-bounds worst-prior
Bayesian regret, although not every prior is hard.
Theorem~\ref{thm:arbitrary-correlated-prior} matches this benchmark in its
$K$-dependent term uniformly over priors, giving the minimax leading order up
to logarithmic factors.

\paragraph{\textbf{All-$K$ comparison.}}
Up to logarithmic factors, the sharp minimax benchmark across all sample sizes
is $\min\{HK,\sqrt{SAH^3K}\}$ \citep{domingues2021,zhang2024}. Theorem~\ref{thm:arbitrary-correlated-prior}
matches the growing $\sqrt{SAH^3K}$ branch uniformly over priors. Its additive
$CS^2AH^2\beta_KL_K$ term comes from the lower-order $1/N$ correction for a row
with $N>0$ visits (Section~\ref{sec:reference-analysis}), with the cumulative calculation in
Section~\ref{subsec:visit-count-aggregation}. Thus the only remaining gap to
the exact all-$K$ curve is in the initial small-$K$ regime, where the $HK$
cap is sharper. Closing that pre-asymptotic gap amounts to sharpening this
lower-order correction; it does not change the minimax leading rate or require
additional prior assumptions.

\subsection{Why Existing Proof Routes Do Not Close the Gap}
\label{subsec:proof-obstacles}

We now make precise the two limitations previewed in the Introduction.
Previous PSRL analyses mainly follow one of two routes. Frequentist-style
arguments replace pathwise optimism by the posterior-sampling identity
\citep{osband2013,osband2017}; information-theoretic arguments use an
information ratio, which compares squared regret with information gain, and
combine it with the complexity of a learning target
\citep{hao2022,moradipari2023}. Neither route directly gives the guarantees
above for one unmodified posterior draw under an arbitrary joint prior.

\paragraph{\textbf{Limits of frequentist-style transfer.}}
\label{subsec:frequentist-transfer-limits}
Frequentist-style PSRL analyses use the posterior-sampling identity to express
regret as a value difference between the sampled and true models. Confidence
sets constructed from past observations cover both models, allowing this
difference to be controlled through local reward and transition errors. The
main challenge is to bound the transition errors sharply enough to obtain
the minimax rate.

Within this framework, \citep{osband2013} use UCRL2-style model confidence
sets \citep{jaksch2010}. Their $\ell_1$ transition radius scales as
$\sqrt{S/N}$; multiplying it by the value range $H$ and summing over visits
gives $\widetilde O(S\sqrt{AH^3K})$. Relative to the homogeneous minimax rate,
uniform transition control introduces a $\sqrt S$ loss, while worst-case
propagation across stages introduces a $\sqrt H$ loss.
\citep{osband2017} sharpen the local control using Gaussian--Dirichlet
concentration, removing the $\sqrt S$ loss and obtaining
$\widetilde O(\sqrt{SAH^3K})$. Their analysis requires independent Dirichlet
transition rows and still handles the $H$ local uncertainties separately,
leaving the homogeneous $\sqrt H$ gap. In stochastic shortest-path models,
\citep{jafarnia2023} use Bernstein confidence sets and posterior calibration
to cover both models. Their coordinate-wise transition bounds retain a
$\sqrt S$ minimax gap, illustrating the limitation of uniform control of the
transition vector.

A sharper frequentist regret bound cannot automatically resolve these
limitations. \citep[Theorem~1]{osband2017} transfer guarantees from OFU
algorithms that optimize over confidence sets of MDPs and whose proofs provide
the corresponding confidence-set error bounds. UCBVI \citep{azar2017} instead
obtains its minimax guarantee through recursive Bellman updates with
empirical-Bernstein bonuses and clipping. Its optimistic values need not be
the optimal values of a single MDP in a confidence set, and its analysis does
not provide the model-error bounds required by that transfer theorem. Thus
UCBVI's minimax guarantee cannot be directly transferred to vanilla PSRL.

\paragraph{\textbf{Limits of information-theoretic analysis.}}
\label{subsec:information-theoretic-limits}
\citep{hao2022} claimed a posterior-sampling guarantee through a
surrogate-environment construction. \citep{moradipari2023} state that this
construction had an incorrect proof and provide a corrected version, so we do
not treat the claimed rate as an established guarantee. \citep{moradipari2023}
give the information-theoretic PSRL analysis most relevant here: they apply the
Russo--Van Roy framework \citep{russo2016,russo2018} to a surrogate learning
target and bound its information ratio and entropy separately. For a target
$Z$ and prior $\rho$, this route gives
$\operatorname{BReg}_{\rho}(K)
\le\sqrt{K\Gamma_{\rho}^{Z}H_{\rho}(Z)}$, where $\Gamma_\rho^Z$ bounds the
information ratio and $H_\rho(Z)$ is the target's entropy under $\rho$.
A prior-uniform conclusion obtained by maximizing the two factors separately
can combine the worst cases of different priors and be polynomially looser
than keeping their product coupled.
Appendix~\ref{app:framework-counterexample} gives an explicit finite-horizon
MDP and two priors for which the separate worst-case envelopes lose a
polynomial factor, even though the same-prior product remains small. This concerns separate worst-case envelopes, not the validity of
\citep{moradipari2023}. Their analysis uses stage-independent priors,
posterior consistency, and a positive-occupancy assumption to bound the
expected information ratio for sufficiently late episodes. Their regret
bound nevertheless includes an additive constant $T_0$ reflecting the onset
of this asymptotic control, which can be large depending on the prior and
problem structure \citep[Theorem~3 and Remark~5]{moradipari2023}.
Their result does not establish a finite-sample, prior-uniform minimax guarantee.

\paragraph{\textbf{Our proof route.}}
\label{subsec:our-proof-route}
\begin{enumerate}
    \item \textbf{Regret decomposition via posterior sampling.}
    The equal conditional distributions of $M^\star$ and $M_k$ reduce regret
    to the value difference of $\pi_k$ between the two models
    (Section~\ref{subsec:posterior-cancellation}).
    \item \textbf{Common empirical reference.}
    Inserting one empirical transition law decomposes each transition error
    into a sample-aligned term, a truth-aligned term, and a single value
    mismatch. This separates the directions that admit scalar control from
    the mismatch requiring uniform control
    (Section~\ref{subsec:empirical-reference}).
    \item \textbf{Local error bounds.}
    Fixed-model scalar concentration transfers to $M^\star$ and $M_k$
    through their common posterior law. Combining it with coordinate-wise
    bounds confines the extra $\sqrt S$ factor to the mismatch
    (Sections~\ref{subsec:posterior-calibration}--\ref{subsec:local-bridge}).
    \item \textbf{Variance control for the value difference.}
    Applying the Bellman square telescope directly to the sampled--true value
    difference bounds its cumulative variance by local errors and regret.
    This sharper control keeps the extra state-space factor out of the
    leading regret term
    (Section~\ref{sec:variance-control}).
\end{enumerate}

Thus the tabular novelty is to use a common empirical reference so that uniform
control is needed only for the value mismatch, and then control that mismatch
by analyzing the value difference directly, rather than bounding the two
value variances separately.
Sections~\ref{sec:reference-analysis}--
\ref{sec:variance-control} implement these steps. Section~\ref{sec:linear-mixture-extension}
adapts the same proof idea to linear-mixture MDPs: instead of inserting an
empirical transition reference, it directly controls the sampled value
directions with a variance-adaptive linear confidence bound.

\section{Tabular Analysis with an Empirical Reference}
\label{sec:reference-analysis}

We first reduce regret to local errors and introduce the empirical-reference
decomposition. We then construct the confidence sets and combine the resulting
bounds for its three terms.
Section~\ref{sec:variance-control} controls the remaining variance terms and
closes the regret bound.

\subsection{Regret Decomposition via Posterior Sampling}
\label{subsec:posterior-cancellation}
\label{subsec:simulation-identity}

As in earlier PSRL analyses \citep{osband2013,osband2017}, the equality
$M_k\mid\mathcal F_k\overset d=M^\star\mid\mathcal F_k$ implies
$\mathbb E_k[V^\star_{1,M^\star}(s_1)-V^\star_{1,M_k}(s_1)]=0$.
Thus the difference between the true and sampled optimal values cancels in
conditional expectation, for arbitrary joint priors. Adding and subtracting
$V^\star_{1,M_k}(s_1)$ in $\Delta_k$ and using $\pi_k=\pi^\star_{M_k}$ gives
\begin{equation}
\label{eq:posterior-sampling-identity}
\mathbb E_k[\Delta_k]
=
\mathbb E_k\!\left[
V^{\pi_k}_{1,M_k}(s_1)-V^{\pi_k}_{1,M^\star}(s_1)
\right].
\end{equation}

Conditional on $(\mathcal F_k,M^\star,M_k)$, consider the trajectory
generated by following $\pi_k$ in the true model $M^\star$.
Along this trajectory, define the one-step reward and transition errors by
\begin{equation}
\label{eq:local-errors}
\begin{aligned}
e^r_{k,h}
&:=r_h^{M_k}(S_{k,h},A_{k,h})
-r_h^{M^\star}(S_{k,h},A_{k,h}),\\
e^P_{k,h}
&:=(P_h^{M_k}-P_h^{M^\star})
V^\star_{h+1,M_k}(S_{k,h},A_{k,h}).
\end{aligned}
\end{equation}
Subtracting the two Bellman recursions and telescoping along the true-model
trajectory yields \citep[Section~5.1]{osband2013}
\begin{equation}
\label{eq:simulation-identity}
V^{\pi_k}_{1,M_k}(s_1)-V^{\pi_k}_{1,M^\star}(s_1)=
\mathbb E\!\left[
\sum_{h=1}^H(e^r_{k,h}+e^P_{k,h})
\,\middle|\,\mathcal F_k,M^\star,M_k
\right].
\end{equation}
The reward error $e^r_{k,h}$ is controlled by standard scalar concentration
and the posterior calibration below; see Lemma~\ref{lem:local-reference-bridge}
for the resulting bound. Its contribution to the final regret bound is
dominated by the transition-error terms. The remaining challenge is to control
the value-weighted transition error $e^P_{k,h}$, because the sampled transition
kernel and its own continuation value are statistically coupled.

\subsection{The Empirical Reference}
\label{subsec:empirical-reference}

Our key new decomposition is to insert one common empirical transition law
between the sampled and true kernels, so that value-uniform control is confined
to a single mismatch term. At the visited row, let
$N_{k,h}:=\sum_{\ell<k}
\mathbf 1\{(S_{\ell,h},A_{\ell,h})=(S_{k,h},A_{k,h})\}$
be its pre-visit count, and let $\widehat p$ be the empirical next-state law
when $N_{k,h}>0$. Evaluating all kernels at $(S_{k,h},A_{k,h})$,
\begin{equation}
\label{eq:reference-decomposition}
\begin{aligned}
e^P_{k,h}
&=
\underbrace{(P_h^{M_k}-\widehat p)V^\star_{h+1,M_k}}
_{\text{sample-aligned}}
+
\underbrace{(\widehat p-P_h^{M^\star})V^\star_{h+1,M^\star}}
_{\text{truth-aligned}}\\
&\quad+
\underbrace{(\widehat p-P_h^{M^\star})
(V^\star_{h+1,M_k}-V^\star_{h+1,M^\star})}
_{\text{value mismatch}}.
\end{aligned}
\end{equation}
The first two terms pair each model's transition kernel with its own optimal
continuation value. For a fixed model, this value vector is deterministic,
so we seek scalar concentration along that direction. The next subsection
constructs the confidence conditions that make this control available for
both the true and sampled models.

The mismatch direction, however, remains data-dependent even when $M^\star$
is fixed, so the same scalar concentration does not directly apply. We use
uniform control for this term, which incurs an additional $\sqrt S$ factor.
The decomposition confines this cost to the mismatch; Section~\ref{sec:variance-control}
then controls its cumulative variance so that the factor does not enter the
leading regret term.

\subsection{Confidence Sets and Concentration Bounds}
\label{subsec:posterior-calibration}

Earlier PSRL analyses use confidence sets to bound the value difference
between the true and sampled models \citep{osband2013,osband2017}.

To control the two aligned terms, we construct confidence sets used only in
the analysis. Fix a row $x=(h,s,a)$ and let
$(R_{x,i},Y_{x,i})$ denote the reward and next state observed on the $i$th
visit to this row. Let $N_k(x)$ count visits to $x$ before episode $k$,
and let $\widehat p_{x,n}$ and
$\widehat r_{x,n}:=n^{-1}\sum_{j=1}^nR_{x,j}$ be the empirical next-state
law and mean reward after $n$ visits. For $n\ge2$, write
\[
\widehat{\operatorname{Var}}_{x,n}(v)
:=\frac{1}{n(n-1)}\sum_{1\le i<j\le n}
\bigl(v(Y_{x,i})-v(Y_{x,j})\bigr)^2.
\]

Define $\mathcal C_k$ as the set of models $M$ satisfying the following
conditions for every $x=(h,s,a)$, $1\le n\le N_k(x)$, and
$i\in\mathcal S$, with the third condition imposed only when $n\ge2$:
\begin{align}
&|\widehat p_{x,n}(i)-P_h^M(i\mid s,a)|
\le
\sqrt{\frac{2P_h^M(i\mid s,a)\beta_K}{n}}
+\frac{2\beta_K}{3n},
\notag\\
&|\widehat r_{x,n}-r_h^M(s,a)|
\le\sqrt{\frac{\beta_K}{2n}},
\notag\\
&\left|(P_h^M-\widehat p_{x,n})V^\star_{h+1,M}\right|\le
C\sqrt{\frac{\beta_K}{n}
\widehat{\operatorname{Var}}_{x,n}(V^\star_{h+1,M})}
+\frac{CH\beta_K}{n}.
\label{eq:candidate-aligned-confidence}
\end{align}
An unobserved row imposes no condition, and $\mathcal C_k$ depends only on
$\mathcal F_k$. For each candidate $M$,
\eqref{eq:candidate-aligned-confidence} compares the average of the observed
values $V^\star_{h+1,M}(Y_{x,j})$ with the expected next-state value predicted
by that model. Under a fixed model $M$, this value vector is deterministic,
so ordinary scalar empirical-Bernstein concentration applies.

\begin{lemma}[Confidence-set coverage]
\label{lem:confidence-set-coverage}
For every fixed model $M$,
\begin{equation}
\label{eq:fixed-model-coverage}
\mathbb P_M\!\left(
M\notin\mathcal C_k\ \text{for some }k\le K
\right)
\le\frac{1}{4K}.
\end{equation}
Under any joint prior, let
$\mathcal G_k:=\{M^\star\in\mathcal C_k\}\cap\{M_k\in\mathcal C_k\}$.
Then
\begin{equation}
\label{eq:failure-budget}
\mathbb P(M_k\notin\mathcal C_k)
=
\mathbb P(M^\star\notin\mathcal C_k),
\qquad
\sum_{k=1}^K\mathbb P(\mathcal G_k^c)\le1.
\end{equation}
\end{lemma}
The fixed-model claim follows from standard concentration and a union bound
over rows, coordinates, and empirical prefixes, not over the model class.
Adaptive visitation is handled by exposing each row's observations in visit
order; Appendix~\ref{app:technical-concentration} gives the details.
The posterior claim follows because $\mathcal C_k$ is determined by
$\mathcal F_k$ and $M_k,M^\star$ have the same conditional law.
Appendix~\ref{app:concentration-calibration} gives this calculation, which
requires no posterior-concentration assumption.

\subsection{Local Reference Bound}
\label{subsec:local-bridge}

On $\mathcal G_k$, both $M_k$ and $M^\star$ satisfy the same history-dependent
confidence conditions relative to the common empirical transition law
$\widehat p$. We can therefore apply the variance-sensitive candidate-aligned
bound from Section~\ref{subsec:posterior-calibration} to each of the two aligned
terms in its own value direction. We use coordinate-wise control
for the remaining mismatch term, as outlined in
Section~\ref{subsec:empirical-reference}.

For a distribution $p$, write
$\operatorname{Var}_p(v):=\sum_i p(i)(v(i)-p^\top v)^2$.
The following elementary lemma supplies that uniform control and also converts
the empirical variance above to the true-row variance. Since the trajectory is
generated by the true model, we express all variance terms under its transition
law to apply the Bellman variance identity in Section~\ref{sec:variance-control}.
The lemma applies to every
$v\in[-H,H]^S$, covering the sampled and true continuation values and their
difference.

\begin{lemma}[Uniform error and variance bounds]
\label{lem:projection-variance-transfer}
If $M\in\mathcal C_k$, then for every row $x=(h,s,a)$, prefix
$1\le n\le N_k(x)$, and $v\in[-H,H]^S$, setting
$p=P_h^M(\cdot\mid s,a)$ and $\widehat p=\widehat p_{x,n}$ gives
\begin{equation}
\label{eq:uniform-projection}
|(\widehat p-p)^\top v|
\le
C\sqrt{\frac{S\beta_K}{n}\operatorname{Var}_p(v)}
+\frac{CSH\beta_K}{n}.
\end{equation}
For $n\ge2$, also
\begin{equation}
\label{eq:variance-transfer}
\frac{n}{n-1}\operatorname{Var}_{\widehat p}(v)
\le
C\left(\operatorname{Var}_p(v)+\frac{SH^2\beta_K}{n}\right).
\end{equation}
\end{lemma}
The complete calculation is given in Appendix~\ref{app:reference-bridge}. For
brevity, write
$\operatorname{Var}_{k,h}(v):=
\operatorname{Var}_{P_h^{M^\star}(\cdot\mid S_{k,h},A_{k,h})}(v)$.

\begin{samepage}
\begin{lemma}[Local reference bound]
\label{lem:local-reference-bridge}
On $\mathcal G_k$, every realized row satisfies
\begin{align}
&|e^P_{k,h}|\le
C\sqrt{\frac{\beta_K}{N_{k,h}\vee1}
\left(\Var_{k,h}(V^\star_{h+1,M_k})
+\Var_{k,h}(V^\star_{h+1,M^\star})\right)}\notag\\
&\qquad+C\sqrt{\frac{S\beta_K}{N_{k,h}\vee1}
\Var_{k,h}
(V^\star_{h+1,M_k}-V^\star_{h+1,M^\star})}
+\frac{CSH\beta_K}{N_{k,h}\vee1},\qquad
|e^r_{k,h}|
\le
C\sqrt{\frac{\beta_K}{N_{k,h}\vee1}}.
\label{eq:local-reference-bridge}
\end{align}
\end{lemma}
\end{samepage}
For $N_{k,h}\ge2$, apply the candidate-aligned bound and
Lemma~\ref{lem:projection-variance-transfer} to the first two terms of
\eqref{eq:reference-decomposition}, combine their square roots using
$\sqrt a+\sqrt b\le\sqrt{2(a+b)}$, and apply
\eqref{eq:uniform-projection} to the mismatch; the small-count cases are
absorbed by the lower-order terms. Appendix~\ref{app:reference-bridge} gives
the complete proof. In \eqref{eq:local-reference-bridge}, the aligned-value term combines
the sample-aligned and truth-aligned contributions, while the
value-mismatch term is the only term carrying the extra $\sqrt S$ factor. It remains to show that its cumulative mismatch variance
has no $KH^2$ baseline, meaning that small local errors and regret lead to a
small variance bound, rather than paying an $H^2$ cost in every episode
regardless of these errors.

\section{Variance Control and Main Proof}
\label{sec:variance-control}

The local reference bound leaves three conditional variance sums. A single
Bellman identity controls all three, and the difference direction is special:
its bound contains no standalone $H^2$ term.

\subsection{A Bellman Variance Identity}
\label{subsec:bellman-square-telescope}

Throughout this and the next subsection, expectations are conditional on fixed
$(\mathcal F_k,M^\star,M_k)$ and are taken over the trajectory generated by
$\pi_k$ in $M^\star$.

\begin{lemma}[Bellman square telescope]
\label{lem:bellman-square-telescope}
Let $v_h:\mathcal S\to\mathbb R$ be fixed under this conditioning, with
$v_{H+1}=0$. Then
\begin{equation}
\mathbb E\sum_{h=1}^H\operatorname{Var}_{k,h}(v_{h+1})
=-v_1(s_1)^2
+\mathbb E\sum_{h=1}^H\Bigl[v_h(S_{k,h})^2-
\bigl(P_h^{M^\star}v_{h+1}(S_{k,h},A_{k,h})\bigr)^2
\Bigr].
\label{eq:bellman-square-telescope}
\end{equation}
Consequently, if $|v_h|\le H$, then
\begin{equation}
\label{eq:bellman-square-consequence}
\mathbb E\sum_{h=1}^H\Var_{k,h}(v_{h+1})
\le 2H\,\mathbb E\sum_{h=1}^H
\Bigl|v_h(S_{k,h})-P_h^{M^\star}v_{h+1}(S_{k,h},A_{k,h})\Bigr|.
\end{equation}
\end{lemma}
The identity follows by expanding each conditional variance and shifting the
index of the squared values: this removes $v_1(s_1)^2$ and adds the zero
terminal term $v_{H+1}(S_{k,H+1})^2$. The consequence uses
$x^2-y^2\le2H|x-y|$ for $|x|,|y|\le H$. This is the residual form of the
Bellman-variance mechanism used in Bernstein-style finite-horizon analyses
\citep{azar2017,zhou2021}.

\subsection{Bounding the Three Variance Terms}
\label{subsec:three-variance-terms}

Applying Lemma~\ref{lem:bellman-square-telescope} in the three value directions
from Section~\ref{subsec:local-bridge} gives the following bounds. The first
two are direct residual applications; the novel step in our analysis is the
third, which controls the sampled--true value difference without a standalone
$H^2$ variance baseline.

\begin{samepage}
\begin{proposition}[Three variance bounds]
\label{prop:three-variance-bounds}
For every episode $k$, with all sums below taken over $h=1,\ldots,H$,
\begin{equation}
\label{eq:three-variance-bounds}
\begin{aligned}
&\mathbb E\sum_h\operatorname{Var}_{k,h}(V^\star_{h+1,M_k})
\le2H^2+2H\,\mathbb E\sum_h|e^P_{k,h}|,\\
&\mathbb E\sum_h\operatorname{Var}_{k,h}(V^\star_{h+1,M^\star})
\le2H^2+2H\Delta_k,\\
&\mathbb E\sum_h\operatorname{Var}_{k,h}
(V^\star_{h+1,M_k}-V^\star_{h+1,M^\star})\le2H\left[
\mathbb E\sum_h(|e^P_{k,h}|+|e^r_{k,h}|)+\Delta_k
\right].
\end{aligned}
\end{equation}
\end{proposition}
\end{samepage}
Each bound follows from \eqref{eq:bellman-square-consequence}, with $v_h$
equal to the sampled value, the true value, or their difference;
Appendix~\ref{app:reference-bridge} expands all three applications.
For the difference direction, its one-step residual under the true
transition is
\[
\begin{aligned}
&\bigl(V^\star_{h,M_k}-V^\star_{h,M^\star}\bigr)(S_{k,h})
-P_h^{M^\star}
\bigl(V^\star_{h+1,M_k}-V^\star_{h+1,M^\star}\bigr)
(S_{k,h},A_{k,h})\\
&\quad=e^r_{k,h}+e^P_{k,h}
-\Bigl[V^\star_{h,M^\star}(S_{k,h})
-r_h^{M^\star}(S_{k,h},A_{k,h})
-P_h^{M^\star}V^\star_{h+1,M^\star}(S_{k,h},A_{k,h})\Bigr].
\end{aligned}
\]
The bracketed term is nonnegative, and its expected sum over the episode is
exactly $\Delta_k$. Applying \eqref{eq:bellman-square-consequence} therefore
gives the third bound. Crucially, no standalone $H^2$ term remains.
Bounding the difference variance by twice the sum of the two individual
variances and using the first two inequalities would retain an $O(H^2)$ term
per episode. After summing over episodes, this $KH^2$ baseline would combine
with the $\sqrt S$ factor in the local mismatch bound from
Section~\ref{subsec:local-bridge}, leaving a $\sqrt S$ gap in the leading
regret term. Applying the telescope directly to the value difference avoids
this baseline and lets the mismatch enter only through the regret self-bound.

\subsection{Closing the Regret Bound}
\label{subsec:visit-count-aggregation}
\label{subsec:proof-main-theorem}

Let $W_P:=\sum_{k,h}\mathbb E[\mathbf1_{\mathcal G_k}|e^P_{k,h}|]$
denote the cumulative transition error and
$W_r:=\sum_{k,h}\mathbb E[\mathbf1_{\mathcal G_k}|e^r_{k,h}|]$
the cumulative reward error.
Posterior sampling and the simulation identity, together with the failure
budget and the visit-count bounds in Appendix~\ref{app:visit-count-sums}, give
\begin{equation}
\label{eq:regret-and-reward}
\operatorname{BReg}_{\rho}(K)\le W_P+W_r+H,
\qquad
W_r\le CH\sqrt{SAK\beta_K}.
\end{equation}
Combining the local reference bound, the visit-count bounds, and
Proposition~\ref{prop:three-variance-bounds} yields
\begin{align}
W_P\le{}&
C\sqrt{SAH^3K\beta_KL_K}
+CS^2AH^2\beta_KL_K\notag\\
&+C\sqrt{S^2AH^2\beta_KL_K
\bigl(W_P+W_r+\operatorname{BReg}_{\rho}(K)\bigr)}.
\label{eq:transition-self-bound}
\end{align}
The key is that Proposition~\ref{prop:three-variance-bounds} controls the
mismatch variance through local errors and regret, without an additional
$KH^2$ term. Substituting \eqref{eq:regret-and-reward} into
\eqref{eq:transition-self-bound} and applying Young's inequality absorbs
the resulting $W_P$ term into the left-hand side. Using the reward bound
then gives
\begin{equation}
\operatorname{BReg}_{\rho}(K)
\le
C\sqrt{SAH^3K\beta_KL_K}
+CS^2AH^2\beta_KL_K.
\label{eq:main-regret-bound}
\end{equation}
Appendix~\ref{app:regret-closure} gives the aggregation and absorption steps.
Thus the extra $\sqrt S$ cost of the value-mismatch direction survives only
in the lower-order term, not in the leading regret rate.

\section{Linear-Mixture Extension}
\label{sec:linear-mixture-extension}

We analyze posterior sampling in time-inhomogeneous linear-mixture MDPs
with known deterministic rewards and an arbitrary joint prior. At stage $h$,
all transition rows share one unknown parameter $\theta_h\in\mathbb R^d$:
\[
\begin{aligned}
P_h^\theta(s'\mid s,a)
&=\langle\phi_h(s,a,s'),\theta_h\rangle,\\
\phi_{h,V}(s,a)
&=\sum_{s'}\phi_h(s,a,s')V(s').
\end{aligned}
\]
The feature maps $\phi_h$ are known. A tuple
$\theta=(\theta_1,\ldots,\theta_H)$ is admissible if it induces valid transition
kernels and satisfies $\|\theta_h\|_2\le B$ and
$\|\phi_{h,V}(s,a)\|_2\le\|V\|_\infty$ for all $h,s,a$ and
$V\in[0,H]^{\mathcal S}$. Let $\rho$ be any joint prior on these tuples,
$\theta^\star\sim\rho$ the true tuple, and
$\theta_k=(\theta_{k,1},\ldots,\theta_{k,H})$ the posterior draw in episode $k$.
Write $V^\star_{h,\theta}$ for the optimal value in model $\theta$.
Appendix~\ref{app:linear-mixture} gives the full construction and proof.

\begin{theorem}[Linear-mixture MDPs]
\label{thm:linear-mixture}
In time-inhomogeneous linear-mixture MDPs under the assumptions above, PSRL
satisfies
\[
\operatorname{BReg}_{\rho}(K)
=\widetilde O\!\left(d\sqrt{H^3K}\right).
\]
\end{theorem}
The bound matches the minimax rate of \citep{zhou2021} and improves the
$\widetilde O(d\sqrt{H^4K})$ guarantee of \citep{li2024}, which assumes
independent stage parameters.

\paragraph{Proof correspondence with the tabular case.}
As in the tabular analysis, posterior sampling reduces regret to local
model errors. In a linear-mixture MDP, the transition error is already a
directional parameter difference,
\[
\left\langle \theta_{k,h}-\theta_h^\star,
\phi_{h,V^\star_{h+1,\theta_k}}(S_{k,h},A_{k,h})\right\rangle.
\]
We control this error by bounding how much candidate parameters disagree
in their predictions along the sampled value direction. A variance-adaptive
confidence width supplies the local control, after which posterior
calibration and Bellman variance closure follow the tabular proof.

\paragraph{Sampled-value direction.}
Before $S_{k,h+1}$ is observed, both $(S_{k,h},A_{k,h})$ and
$V^\star_{h+1,\theta_k}$ are already fixed. For an admissible $\theta$,
define the normalized feature and noise variance
\[
\begin{aligned}
x_{k,h}
&:=\frac{\phi_{h,V^\star_{h+1,\theta_k}}(S_{k,h},A_{k,h})}{H},\\
\sigma_{k,h}^2(\theta)
&:=\frac{\operatorname{Var}_{P_h^\theta}
(V^\star_{h+1,\theta_k})}{H^2},
\end{aligned}
\]
where the variance is evaluated at $(S_{k,h},A_{k,h})$. Then
\[
\mathbb E\!\left[
\frac{V^\star_{h+1,\theta_k}(S_{k,h+1})}{H}
\,\middle|\,\mathcal F_k,\theta^\star,\theta_k,S_{k,1:h},A_{k,1:h}\right]=
\langle\theta_h^\star,x_{k,h}\rangle .
\]
A candidate $\vartheta\in\mathbb R^d$ for the stage parameter $\theta_h$
therefore predicts the normalized next-state value by
$\langle\vartheta,x_{k,h}\rangle$. Two candidates
disagree by $|\langle\vartheta-\vartheta',x_{k,h}\rangle|$; the confidence
width is the largest such disagreement among candidates in the set.

\paragraph{Variance-adaptive confidence width.}
We seek to bound the cumulative disagreement using the actual noise
variances $\sigma_{k,h}^2(\theta)$. This retains the information needed for
Bellman variance closure, which would be lost by replacing every variance
with its worst-case bound of one. We adapt the concentration and weighting
tools of \citep{zhao2023} to obtain the following result. Write
$\Lambda_K=c[1+\log(c d H K(1+B))]$ for a sufficiently large universal $c$.

\begin{lemma}[Sampled-value confidence width]
\label{lem:linear-sampled-width}
There exist confidence sets $\mathcal C_{k,h}$ based only on the history before
 episode $k$, without using $\theta_k$, such that for every fixed admissible
$\theta$, with
probability at least $1-(4HK)^{-2}$,
$\theta_h\in\mathcal C_{k,h}$ for all $k,h$ and, simultaneously for all $h$,
\begin{equation}
\sum_{k=1}^K
\Bigl[1\wedge\sup_{\vartheta,\vartheta'\in\mathcal C_{k,h}}
|\langle\vartheta-\vartheta',x_{k,h}\rangle|\Bigr]
\le Cd\Lambda_K^4
\biggl(\sqrt{\sum_{k=1}^K\sigma_{k,h}^2(\theta)}+1\biggr).
\label{lin:eq:main-width}
\end{equation}
\end{lemma}

Appendix~\ref{lin:subsec:width-construction} gives the confidence-set
construction and proof.

\paragraph{Applying the width bound to PSRL.}
The sets use only the history before the fresh posterior draw. Conditional
on this history, the true and sampled parameters have the same law, so
posterior calibration transfers coverage to the sample without any
independence assumption on the prior. When both parameters lie in the set,
the transition Bellman error
\[
e^P_{k,h}:=H\langle\theta_{k,h}-\theta_h^\star,x_{k,h}\rangle
\]
is bounded in absolute value by $H$ times the clipped width in
\eqref{lin:eq:main-width}.

\paragraph{Cross-model variance closure.}
For $\theta=\theta^\star$, the normalized variance
$\sigma_{k,h}^2(\theta^\star)$ in \eqref{lin:eq:main-width} corresponds to
$\operatorname{Var}_{P_h^{\theta^\star}}(V^\star_{h+1,\theta_k})$: the
transition comes from the true model while the value comes from the sampled
model. UCRL-VTR$^+$ is a frequentist OFU algorithm whose variance analysis
relies on optimistic value estimates \citep[Lemma~21]{zhou2021}.
Posterior-sampled values need not be optimistic, so we instead control this
cross-model variance directly through the Bellman square telescope.

\begin{proposition}[Cross-model variance closure]
\label{prop:linear-cross-model-variance}
Let
$W_P:=\sum_{k,h}\mathbb E|e^P_{k,h}|$. Then
\[
\sum_{k,h}\mathbb E\,
\operatorname{Var}_{P_h^{\theta^\star}}
(V^\star_{h+1,\theta_k})
\le 2H^2K+2H W_P.
\]
\end{proposition}
Thus the cross-model variance is controlled by the same transition errors
that appear in the regret, giving the linear-mixture counterpart of the
variance closure in Section~\ref{sec:variance-control}.

\paragraph{Closing the self-bound.}
Combining Lemma~\ref{lem:linear-sampled-width} with
Proposition~\ref{prop:linear-cross-model-variance} and summing over stages gives
\[
W_P
\le
C d\sqrt{H^3K}\,\Lambda_K^4
+C dH \Lambda_K^4\sqrt{W_P}
+C dH^2\Lambda_K^4.
\]
Young's inequality absorbs the middle term. For large enough $K$ this yields
the leading $\widetilde O(d\sqrt{H^3K})$ rate directly; for smaller $K$, the
trivial bound $\operatorname{BReg}_\rho(K)\le HK$ is already no larger than
the same order. This proves Theorem~\ref{thm:linear-mixture}. Direct candidate
tests provide the local width; posterior calibration and Bellman-square
control of the cross-model variance complete the PSRL argument.

\section{Discussion and Limitations}
\label{sec:discussion}

We showed that exact vanilla PSRL matches the minimax leading Bayesian-regret
rate under arbitrary joint priors in time-inhomogeneous tabular MDPs and
attains the minimax rate in linear-mixture MDPs. In the tabular case, a common
empirical transition reference confines value-uniform control to one
sampled--true value mismatch, which is then closed by a Bellman variance
argument. In linear-mixture MDPs, variance-adaptive control in the sampled
value directions and the same Bellman variance closure play the corresponding
roles. Together, these arguments resolve the dependence between a
posterior-sampled model and its own value function without requiring prior
factorization or posterior regularity.

The tabular bound retains a lower-order finite-sample correction and therefore
does not match the exact minimax curve uniformly over all $K$, while the
linear-mixture result assumes known deterministic rewards and standard
boundedness conditions. These restrictions do not affect the minimax leading
rates, but they clarify the scope of the current guarantees. The appendices
provide homogeneous specializations and model-independent initial-state
extensions; natural further directions include removing the tabular
correction and extending the analysis to approximate inference, approximate
planning, and richer structured model classes.

\bibliographystyle{plainnat}
\bibliography{references}

@article{jaksch2010,
  title   = {Near-Optimal Regret Bounds for Reinforcement Learning},
  author  = {Jaksch, Thomas and Ortner, Ronald and Auer, Peter},
  journal = {Journal of Machine Learning Research},
  volume  = {11},
  pages   = {1563--1600},
  year    = {2010}
}

@inproceedings{agrawaljia2017,
  title     = {Optimistic Posterior Sampling for Reinforcement Learning: Worst-Case Regret Bounds},
  author    = {Agrawal, Shipra and Jia, Randy},
  booktitle = {Advances in Neural Information Processing Systems},
  volume    = {30},
  pages     = {1184--1194},
  year      = {2017}
}

@inproceedings{azar2017,
  title     = {Minimax Regret Bounds for Reinforcement Learning},
  author    = {Azar, Mohammad Gheshlaghi and Osband, Ian and Munos, R{\'e}mi},
  booktitle = {Proceedings of the 34th International Conference on Machine Learning},
  series    = {Proceedings of Machine Learning Research},
  volume    = {70},
  pages     = {263--272},
  year      = {2017}
}

@inproceedings{ayoub2020,
  title     = {Model-Based Reinforcement Learning with Value-Targeted Regression},
  author    = {Ayoub, Alex and Jia, Zeyu and Szepesv{\'a}ri, Csaba and Wang, Mengdi and Yang, Lin},
  booktitle = {Proceedings of the 37th International Conference on Machine Learning},
  series    = {Proceedings of Machine Learning Research},
  volume    = {119},
  pages     = {463--474},
  year      = {2020}
}

@inproceedings{fanming2021,
  title     = {Model-Based Reinforcement Learning for Continuous Control with Posterior Sampling},
  author    = {Fan, Ying and Ming, Yifei},
  booktitle = {Proceedings of the 38th International Conference on Machine Learning},
  series    = {Proceedings of Machine Learning Research},
  volume    = {139},
  pages     = {3078--3087},
  year      = {2021}
}

@book{boucheron2013,
  title     = {Concentration Inequalities: A Nonasymptotic Theory of Independence},
  author    = {Boucheron, St{\'e}phane and Lugosi, G{\'a}bor and Massart, Pascal},
  publisher = {Oxford University Press},
  year      = {2013}
}

@inproceedings{dann2021,
  title     = {A Provably Efficient Model-Free Posterior Sampling Method for Episodic Reinforcement Learning},
  author    = {Dann, Christoph and Mohri, Mehryar and Zhang, Tong and Zimmert, Julian},
  booktitle = {Advances in Neural Information Processing Systems},
  volume    = {34},
  pages     = {12040--12051},
  year      = {2021}
}

@inproceedings{domingues2021,
  title     = {Episodic Reinforcement Learning in Finite {MDPs}: Minimax Lower Bounds Revisited},
  author    = {Domingues, Omar Darwiche and M{\'e}nard, Pierre and Kaufmann, Emilie and Valko, Michal},
  booktitle = {Proceedings of the 32nd International Conference on Algorithmic Learning Theory},
  series    = {Proceedings of Machine Learning Research},
  volume    = {132},
  pages     = {578--598},
  year      = {2021}
}

@inproceedings{hao2022,
  title     = {Regret Bounds for Information-Directed Reinforcement Learning},
  author    = {Hao, Botao and Lattimore, Tor},
  booktitle = {Advances in Neural Information Processing Systems},
  volume    = {35},
  pages     = {28575--28587},
  year      = {2022}
}

@inproceedings{hong2022,
  title     = {{Thompson} Sampling with a Mixture Prior},
  author    = {Hong, Joey and Kveton, Branislav and Zaheer, Manzil and Ghavamzadeh, Mohammad and Boutilier, Craig},
  booktitle = {Proceedings of the 25th International Conference on Artificial Intelligence and Statistics},
  series    = {Proceedings of Machine Learning Research},
  volume    = {151},
  pages     = {7565--7586},
  year      = {2022}
}

@inproceedings{jafarnia2023,
  title     = {Posterior Sampling-Based Online Learning for the Stochastic Shortest Path Model},
  author    = {Jafarnia-Jahromi, Mehdi and Chen, Liyu and Jain, Rahul and Luo, Haipeng},
  booktitle = {Proceedings of the Thirty-Ninth Conference on Uncertainty in Artificial Intelligence},
  series    = {Proceedings of Machine Learning Research},
  volume    = {216},
  pages     = {922--931},
  year      = {2023}
}

@inproceedings{li2024,
  title     = {Prior-Dependent Analysis of Posterior Sampling Reinforcement Learning with Function Approximation},
  author    = {Li, Yingru and Luo, Zhiquan},
  booktitle = {Proceedings of the 27th International Conference on Artificial Intelligence and Statistics},
  series    = {Proceedings of Machine Learning Research},
  volume    = {238},
  pages     = {559--567},
  year      = {2024}
}

@inproceedings{lu2019,
  title     = {Information-Theoretic Confidence Bounds for Reinforcement Learning},
  author    = {Lu, Xiuyuan and Van Roy, Benjamin},
  booktitle = {Advances in Neural Information Processing Systems},
  volume    = {32},
  pages     = {2458--2466},
  year      = {2019}
}

@inproceedings{maurer2009,
  title     = {Empirical {Bernstein} Bounds and Sample-Variance Penalization},
  author    = {Maurer, Andreas and Pontil, Massimiliano},
  booktitle = {Proceedings of the 22nd Annual Conference on Learning Theory},
  year      = {2009}
}

@inproceedings{moradipari2023,
  title     = {Improved {Bayesian} Regret Bounds for {Thompson} Sampling in Reinforcement Learning},
  author    = {Moradipari, Ahmadreza and Pedramfar, Mohammad and Shokrian Zini, Modjtaba and Aggarwal, Vaneet},
  booktitle = {Advances in Neural Information Processing Systems},
  volume    = {36},
  pages     = {23557--23569},
  year      = {2023}
}

@inproceedings{osband2013,
  title     = {{(More)} Efficient Reinforcement Learning via Posterior Sampling},
  author    = {Osband, Ian and Russo, Daniel and Van Roy, Benjamin},
  booktitle = {Advances in Neural Information Processing Systems},
  volume    = {26},
  pages     = {3003--3011},
  year      = {2013}
}

@inproceedings{osband2014,
  title     = {Model-Based Reinforcement Learning and the Eluder Dimension},
  author    = {Osband, Ian and Van Roy, Benjamin},
  booktitle = {Advances in Neural Information Processing Systems},
  volume    = {27},
  pages     = {1466--1474},
  year      = {2014}
}

@inproceedings{osband2017,
  title     = {Why Is Posterior Sampling Better than Optimism for Reinforcement Learning?},
  author    = {Osband, Ian and Van Roy, Benjamin},
  booktitle = {Proceedings of the 34th International Conference on Machine Learning},
  series    = {Proceedings of Machine Learning Research},
  volume    = {70},
  pages     = {2701--2710},
  year      = {2017}
}

@article{russo2018,
  title   = {Learning to Optimize via Information-Directed Sampling},
  author  = {Russo, Daniel and Van Roy, Benjamin},
  journal = {Operations Research},
  volume  = {66},
  number  = {1},
  pages   = {230--252},
  year    = {2018}
}

@article{russo2016,
  title   = {An Information-Theoretic Analysis of {Thompson} Sampling},
  author  = {Russo, Daniel and Van Roy, Benjamin},
  journal = {Journal of Machine Learning Research},
  volume  = {17},
  number  = {68},
  pages   = {1--30},
  year    = {2016}
}

@inproceedings{strens2000,
  title     = {A {Bayesian} Framework for Reinforcement Learning},
  author    = {Strens, Malcolm J. A.},
  booktitle = {Proceedings of the 17th International Conference on Machine Learning},
  pages     = {943--950},
  year      = {2000}
}

@inproceedings{tiapkin2022bayesucb,
  title     = {From {Dirichlet} to {Rubin}: Optimistic Exploration in {RL} without Bonuses},
  author    = {Tiapkin, Daniil and Belomestny, Denis and Moulines, {\'E}ric and Naumov, Alexey and Samsonov, Sergey and Tang, Yunhao and Valko, Michal and M{\'e}nard, Pierre},
  booktitle = {Proceedings of the 39th International Conference on Machine Learning},
  series    = {Proceedings of Machine Learning Research},
  volume    = {162},
  pages     = {21380--21431},
  year      = {2022}
}

@inproceedings{tiapkin2022opsrl,
  title     = {Optimistic Posterior Sampling for Reinforcement Learning with Few Samples and Tight Guarantees},
  author    = {Tiapkin, Daniil and Belomestny, Denis and Calandriello, Daniele and Moulines, {\'E}ric and Munos, R{\'e}mi and Naumov, Alexey and Rowland, Mark and Valko, Michal and M{\'e}nard, Pierre},
  booktitle = {Advances in Neural Information Processing Systems},
  volume    = {35},
  pages     = {10737--10751},
  year      = {2022}
}

@article{thompson1933,
  title   = {On the Likelihood that One Unknown Probability Exceeds Another in View of the Evidence of Two Samples},
  author  = {Thompson, William R.},
  journal = {Biometrika},
  volume  = {25},
  number  = {3/4},
  pages   = {285--294},
  year    = {1933}
}

@article{zhang2022feelgood,
  title   = {Feel-Good {Thompson} Sampling for Contextual Bandits and Reinforcement Learning},
  author  = {Zhang, Tong},
  journal = {SIAM Journal on Mathematics of Data Science},
  volume  = {4},
  number  = {2},
  pages   = {834--857},
  year    = {2022}
}

@inproceedings{zhang2024,
  title     = {Settling the Sample Complexity of Online Reinforcement Learning},
  author    = {Zhang, Zihan and Chen, Yuxin and Lee, Jason D. and Du, Simon S.},
  booktitle = {Proceedings of the 37th Conference on Learning Theory},
  series    = {Proceedings of Machine Learning Research},
  volume    = {247},
  pages     = {5213--5219},
  year      = {2024}
}

@inproceedings{zhou2021,
  title     = {Nearly Minimax Optimal Reinforcement Learning for Linear Mixture {Markov} Decision Processes},
  author    = {Zhou, Dongruo and Gu, Quanquan and Szepesv{\'a}ri, Csaba},
  booktitle = {Proceedings of the 34th Conference on Learning Theory},
  series    = {Proceedings of Machine Learning Research},
  volume    = {134},
  pages     = {4532--4576},
  year      = {2021}
}

@inproceedings{zhao2023,
  title     = {Variance-Dependent Regret Bounds for Linear Bandits and Reinforcement Learning: Adaptivity and Computational Efficiency},
  author    = {Zhao, Heyang and He, Jiafan and Zhou, Dongruo and Zhang, Tong and Gu, Quanquan},
  booktitle = {Proceedings of the 36th Conference on Learning Theory},
  series    = {Proceedings of Machine Learning Research},
  volume    = {195},
  pages     = {4977--5020},
  year      = {2023}
}

@inproceedings{zhougu2022,
  title     = {Computationally Efficient Horizon-Free Reinforcement Learning for Linear Mixture {MDP}s},
  author    = {Zhou, Dongruo and Gu, Quanquan},
  booktitle = {Advances in Neural Information Processing Systems},
  volume    = {35},
  pages     = {36337--36349},
  year      = {2022}
}

@inproceedings{agarwalcmdp2022,
  title     = {Regret Guarantees for Model-Based Reinforcement Learning with Long-Term Average Constraints},
  author    = {Agarwal, Mridul and Bai, Qinbo and Aggarwal, Vaneet},
  booktitle = {Proceedings of the Thirty-Eighth Conference on Uncertainty in Artificial Intelligence},
  series    = {Proceedings of Machine Learning Research},
  volume    = {180},
  pages     = {22--31},
  year      = {2022}
}

@inproceedings{bayrooti2025,
  title     = {No-Regret {Thompson} Sampling for Finite-Horizon {Markov} Decision Processes with {Gaussian} Processes},
  author    = {Bayrooti, Jasmine and Vakili, Sattar and Prorok, Amanda and Ek, Carl Henrik},
  booktitle = {Advances in Neural Information Processing Systems},
  volume    = {38},
  pages     = {81043--81071},
  doi       = {10.52202/085713-2443},
  year      = {2025}
}

@article{robert2025,
  title   = {Posterior Sampling for Reinforcement Learning on Graphs},
  author  = {Robert, Arnaud and Faisal, Aldo A. and Pike-Burke, Ciara},
  journal = {Transactions on Machine Learning Research},
  url     = {https://openreview.net/forum?id=kd6CfmdPfX},
  year    = {2025}
}

@inproceedings{agnihotri2026,
  title     = {Best Policy Learning from Trajectory Preference Feedback},
  author    = {Agnihotri, Akhil and Jain, Rahul and Ramachandran, Deepak and Wen, Zheng},
  booktitle = {Proceedings of the 29th International Conference on Artificial Intelligence and Statistics},
  series    = {Proceedings of Machine Learning Research},
  volume    = {300},
  pages     = {2116--2124},
  year      = {2026}
}

@inproceedings{flynn2026,
  title     = {Posterior Sampling Reinforcement Learning with {Gaussian} Processes for Continuous Control: Sublinear Regret Bounds for Unbounded State Spaces},
  author    = {Flynn, Hamish and Watson, Joe and Posner, Ingmar and Peters, Jan},
  booktitle = {Proceedings of the 43rd International Conference on Machine Learning},
  note      = {Accepted; preprint arXiv:2603.08287},
  url       = {https://arxiv.org/abs/2603.08287},
  year      = {2026}
}

\clearpage
\appendix

\let\addtocontents\appendixaddtocontents
\setcounter{tocdepth}{2}
\begingroup
\renewcommand{\contentsname}{Contents of the Appendix}
\hypersetup{linktoc=all}
\pdfbookmark[1]{Contents of the Appendix}{appendix-contents}
\tableofcontents
\endgroup
\clearpage

\section{Proof of the Tabular Result}
\label{app:tabular-proof}

We give a complete proof of Theorem~\ref{thm:arbitrary-correlated-prior},
including the model and algorithm conventions, regret identities, confidence-set
construction, local error and variance bounds, and visit-count aggregation.
Each result is stated before its proof; results from the main text retain
their original numbers.

\paragraph{Setting and notation.}
Let $|\mathcal S|=S$, $|\mathcal A|=A$, and let each episode have $H$
stages and initial state $s_1$. A model $M$ specifies a joint law
$J_h^M(\cdot\mid s,a)$ on $[0,1]\times\mathcal S$ for every $(h,s,a)$.
Conditional on $M$, each visit produces a fresh reward--next-state draw from
this law. Write $\nu_h^M$ and $P_h^M$ for its reward and transition marginals,
$r_h^M(s,a):=\mathbb E_M[R_h\mid s,a]$, and
$P_h^Mv(s,a):=\sum_{s'}P_h^M(s'\mid s,a)v(s')$.
For a stage-dependent Markov policy $\pi$, put
\[
V^\pi_{h,M}(s):=\mathbb E_{M,\pi}\!\left[\sum_{t=h}^H R_t\,\middle|\,S_h=s\right],
\qquad V^\star_{h,M}(s):=\max_\pi V^\pi_{h,M}(s),
\qquad V^\pi_{H+1,M}=V^\star_{H+1,M}=0.
\]
Thus all value functions lie in $[0,H]$.

Draw $M^\star\sim\rho$ from an arbitrary joint prior on these row laws.
The history $\mathcal F_k$ contains all observations and algorithmic
randomness before episode $k$. With
$\rho_k:=\mathbb P(M^\star\in\cdot\mid\mathcal F_k)$,
PSRL draws $M_k\sim\rho_k$ using fresh randomness and follows an optimal
policy $\pi_k=\pi^\star_{M_k}$ throughout the episode. Let
$(S_{k,h},A_{k,h})$ denote this trajectory in $M^\star$, and define
\[
\mathbb E_k[\cdot]:=\mathbb E[\cdot\mid\mathcal F_k],\qquad
\Delta_k:=V^\star_{1,M^\star}(s_1)-V^{\pi_k}_{1,M^\star}(s_1),\qquad
\operatorname{BReg}_\rho(K):=\mathbb E\sum_{k=1}^K\Delta_k.
\]
The expectation in Bayesian regret includes the prior draw, posterior draws,
and trajectories. Set
\[
L_K:=2+\log K,\qquad \beta_K:=\log(c_0S^2AHK^2),
\]
where $c_0$ is sufficiently large. The symbol $C$ denotes a universal positive
constant that may change between displays.

\paragraph{Prior and conditional-law conventions.}
Equip the space of probability measures on $[0,1]\times\mathcal S$ with the
Borel $\sigma$-field generated by weak convergence. Each row law
$J_h^M(\cdot\mid s,a)$ is an element of this space, and the finite product
over $(h,s,a)$ is standard Borel. An arbitrary prior is any Borel probability
measure $\rho$ on this space; no density, domination, factorization, or
posterior-consistency condition is assumed. Each finite history, including
previous posterior draws and algorithmic randomness, takes values in a
standard Borel space, and its law is a measurable kernel of $M$. Hence a
regular conditional posterior
$\rho_k=\mathbb P(M^\star\in\cdot\mid\mathcal F_k)$ exists. We fix one version
and interpret all conditional identities almost surely. On prior-predictive
null histories, this version and the sampled policy may be chosen arbitrarily;
the fixed-model concentration below holds for any such adaptive choices, and
Bayesian regret is unchanged.

Fresh algorithmic randomness draws $M_k$ from $\rho_k$ independently of
$M^\star$ given $\mathcal F_k$. Fixing an order on the finite action set and
using it to break every Bellman maximization makes
$M\mapsto V_{h,M}^\star$ and $M\mapsto\pi_M^\star$ measurable by backward
induction. All values lie in $[0,H]$, so the conditional expectations below
are finite. The same convention applies to the linear-mixture parameter class
satisfying the validity and boundedness conditions stated in
Appendix~\ref{app:linear-mixture}, viewed as a Borel subset of a
finite-dimensional Euclidean space.

\subsection{Regret Decomposition}
\label{app:technical-identities}

At each visited row, define
\[
\begin{aligned}
e^r_{k,h}&:=r_h^{M_k}(S_{k,h},A_{k,h})-r_h^{M^\star}(S_{k,h},A_{k,h}),\\
e^P_{k,h}&:=(P_h^{M_k}-P_h^{M^\star})V^\star_{h+1,M_k}(S_{k,h},A_{k,h}).
\end{aligned}
\]

\begin{lemma}[Posterior-sampling and simulation identities]
\label{lem:posterior-simulation-identities}
For every episode $k$,
\[
\mathbb E_k[\Delta_k]
=\mathbb E_k\!\left[V^{\pi_k}_{1,M_k}(s_1)-V^{\pi_k}_{1,M^\star}(s_1)\right],
\]
and, conditional on $(\mathcal F_k,M^\star,M_k)$,
\[
V^{\pi_k}_{1,M_k}(s_1)-V^{\pi_k}_{1,M^\star}(s_1)
=\mathbb E\!\left[\sum_{h=1}^H(e^r_{k,h}+e^P_{k,h})
\,\middle|\,\mathcal F_k,M^\star,M_k\right].
\]
\end{lemma}

\begin{proof}
\emph{Posterior sampling.}
Conditional on $\mathcal F_k$,
\[
M^\star\sim\rho_k,
\qquad
M_k\sim\rho_k,
\qquad
\pi_k=\pi_{M_k}^\star.
\]
Therefore
\[
\begin{aligned}
\mathbb E_k[V^\star_{1,M^\star}(s_1)]
&=\mathbb E_k[V^\star_{1,M_k}(s_1)]
=\mathbb E_k[V^{\pi_k}_{1,M_k}(s_1)],\\
\mathbb E_k[\Delta_k]
&=\mathbb E_k[V^\star_{1,M^\star}(s_1)]
-\mathbb E_k[V^{\pi_k}_{1,M^\star}(s_1)]
=\mathbb E_k[V^{\pi_k}_{1,M_k}(s_1)
-V^{\pi_k}_{1,M^\star}(s_1)].
\end{aligned}
\]
No componentwise equality of $M^\star$ and $M_k$ is used.

\emph{Simulation.}
Following \citep[Section~5.1]{osband2013}, fix
$(\mathcal F_k,M^\star,M_k)$. Since $\pi_k$ is optimal in $M_k$, for
$a=\pi_{k,h}(s)$,
\[
\begin{aligned}
V^{\pi_k}_{h,M_k}(s)-V^{\pi_k}_{h,M^\star}(s)
&={}r_h^{M_k}(s,a)-r_h^{M^\star}(s,a)
+(P_h^{M_k}-P_h^{M^\star})V^\star_{h+1,M_k}(s,a)
\\&\quad+P_h^{M^\star}
(V^{\pi_k}_{h+1,M_k}-V^{\pi_k}_{h+1,M^\star})(s,a).
\end{aligned}
\]
At $(S_{k,h},A_{k,h})$,
\[
\begin{aligned}
&\mathbb E\!\left[
V^{\pi_k}_{h,M_k}(S_{k,h})-V^{\pi_k}_{h,M^\star}(S_{k,h})
\mid\mathcal F_k,M^\star,M_k\right]\\
&\quad=\mathbb E\!\left[e^r_{k,h}+e^P_{k,h}
\mid\mathcal F_k,M^\star,M_k\right]
+\mathbb E\!\left[
V^{\pi_k}_{h+1,M_k}(S_{k,h+1})
-V^{\pi_k}_{h+1,M^\star}(S_{k,h+1})
\mid\mathcal F_k,M^\star,M_k\right].
\end{aligned}
\]
Summing over $h$, using $S_{k,1}=s_1$ and the zero terminal values,
gives the second identity.
\end{proof}

\subsection{Confidence-Set Construction and Coverage}
\label{app:concentration-calibration}
\label{app:technical-concentration}

Write $\mathbb P_M$ for the law of the observations and algorithmic
randomness when the environment is fixed to $M$.
For a row $x=(h,s,a)$, expose its observations in visit order as
$(R_{x,j},Y_{x,j})_{j\ge1}$ and define
\[
N_k(x):=\sum_{\ell<k}
\mathbf1\{(S_{\ell,h},A_{\ell,h})=(s,a)\},
\quad
\widehat r_{x,n}:=\frac1n\sum_{j=1}^nR_{x,j},
\quad
\widehat p_{x,n}(i):=\frac1n\sum_{j=1}^n\mathbf1\{Y_{x,j}=i\}.
\]
For $n\ge2$, define the empirical variance by
\[
\widehat{\operatorname{Var}}_{x,n}(v)
:=\frac{1}{n(n-1)}\sum_{1\le i<j\le n}
\bigl(v(Y_{x,i})-v(Y_{x,j})\bigr)^2.
\]
Define $\mathcal C_k$ as the set of candidate models $M$ satisfying, for every row
$x=(h,s,a)$, every prefix $1\le n\le N_k(x)$, and every $i\in\mathcal S$,
\begin{align}
|\widehat p_{x,n}(i)-P_h^M(i\mid s,a)|
&\le
\sqrt{\frac{2P_h^M(i\mid s,a)\beta_K}{n}}
+\frac{2\beta_K}{3n},
\label{eq:coordinate-confidence}\\
|\widehat r_{x,n}-r_h^M(s,a)|
&\le \sqrt{\frac{\beta_K}{2n}},
\label{eq:reward-confidence}
\end{align}
and, whenever $n\ge2$,
\begin{equation}
\left|(P_h^M-\widehat p_{x,n})V^\star_{h+1,M}\right|
\le
C\sqrt{\frac{\beta_K}{n}
\widehat{\operatorname{Var}}_{x,n}(V^\star_{h+1,M})}
+\frac{CH\beta_K}{n}.
\label{eq:complete-candidate-aligned-confidence}
\end{equation}
An unobserved row imposes no condition, and for $n=1$ only
\eqref{eq:coordinate-confidence}--\eqref{eq:reward-confidence} are imposed;
the candidate-aligned empirical-variance test starts at $n=2$. Thus
$\mathcal C_k$ is determined entirely by $\mathcal F_k$. The omitted
candidate-aligned test at counts zero and one is harmless because those visits
are handled directly by the small-count part of
Lemma~\ref{lem:local-reference-bridge}.

The constant $C$ in
\eqref{eq:complete-candidate-aligned-confidence} is fixed once and for all to a valid
universal constant in the empirical Bernstein inequality.

\begin{restatement}{Lemma}{lem:confidence-set-coverage}{Confidence-set coverage}
For every fixed model $M$,
\[
\mathbb P_M\!\left(
M\notin\mathcal C_k\ \text{for some }k\le K
\right)
\le\frac{1}{4K}.
\]
Under any joint prior, let
$\mathcal G_k:=\{M^\star\in\mathcal C_k\}\cap\{M_k\in\mathcal C_k\}$.
Then
\[
\mathbb P(M_k\notin\mathcal C_k)
=
\mathbb P(M^\star\notin\mathcal C_k),
\qquad
\sum_{k=1}^K\mathbb P(\mathcal G_k^c)\le1.
\]
\end{restatement}

\begin{proof}
\emph{Fixed-model coverage.}
Fix $M$. For each row $x=(h,s,a)$, expose an infinite sequence
$(R_{x,j},Y_{x,j})_{j\ge1}$ of independent draws from
$J_h^M(\cdot\mid s,a)$. Adaptive visitation changes only which row reveals its
next unused pair. Thus, for every fixed $x$, $i\in\mathcal S$, and $n\le K$,
\[
\begin{gathered}
\mathbf1\{Y_{x,1}=i\},\ldots,\mathbf1\{Y_{x,n}=i\}
\stackrel{\mathrm{iid}}{\sim}
\mathrm{Bernoulli}(P_h^M(i\mid s,a)),\\
R_{x,1},\ldots,R_{x,n}
\stackrel{\mathrm{iid}}{\sim}\nu_h^M(\cdot\mid s,a),\\
V^\star_{h+1,M}(Y_{x,1}),\ldots,V^\star_{h+1,M}(Y_{x,n})
\text{ are i.i.d. in }[0,H],\\
\mathbb E_M[V^\star_{h+1,M}(Y_{x,j})]
=P_h^M V^\star_{h+1,M}(s,a).
\end{gathered}
\]
Since $M$ is fixed, $V^\star_{h+1,M}$ is deterministic in these
concentration inequalities. If
$v(Y_{x,1}),\ldots,v(Y_{x,n})$ are i.i.d. in $[0,H]$ under a fixed model and
$n\ge2$, empirical Bernstein gives \citep{maurer2009,boucheron2013}, with
probability at least $1-Ce^{-\beta_K}$,
\[
\begin{aligned}
\left|P_h^M v(s,a)-\frac1n\sum_{j=1}^n v(Y_{x,j})\right|
&\le C\sqrt{\frac{\beta_K}{n^2(n-1)}
\sum_{i<j}(v(Y_{x,i})-v(Y_{x,j}))^2}
+\frac{CH\beta_K}{n}\\
&=C\sqrt{\frac{\beta_K}{n}
\widehat{\operatorname{Var}}_{x,n}(v)}
+\frac{CH\beta_K}{n}.
\end{aligned}
\]
Here $\beta_K=\log(c_0S^2AHK^2)$ is chosen so that a union bound over all
rows, coordinates, and prefixes has failure probability at most $1/(4K)$
for sufficiently large $c_0$.
Setting $v=V^\star_{h+1,M}$ gives
\eqref{eq:complete-candidate-aligned-confidence}.
Applying Bernstein's inequality to the next-state indicators
(whose variance is at most $P_h^M(i\mid s,a)$), Hoeffding's inequality to the
$[0,1]$-valued rewards, and the empirical Bernstein bound above to
$V^\star_{h+1,M}(Y_{x,j})$ gives
\[
\begin{aligned}
\mathbb P_M(\text{coordinate condition fails at }x,i,n)
&\le2e^{-\beta_K},\\
\mathbb P_M(\text{reward condition fails at }x,n)
&\le2e^{-\beta_K},\\
\mathbb P_M(\text{candidate-aligned condition fails at }x,n)
&\le Ce^{-\beta_K}.
\end{aligned}
\]
A finite union bound therefore gives
\[
\mathbb P_M\!\left(M\notin\mathcal C_k
\text{ for some }k\le K\right)
\le C(SAH)(S+2)K e^{-\beta_K}
\le CS^2AHK e^{-\beta_K}
\le\frac1{4K},
\]
after increasing $c_0$. No union bound over the model space is used.

Within a fixed row,
\[
(R_{x,j},Y_{x,j})_{j\ge1}\stackrel{\mathrm{iid}}{\sim}
J_h^M(\cdot\mid s,a),
\qquad
R_{x,j}\sim\nu_h^M(\cdot\mid s,a),
\qquad
Y_{x,j}\sim P_h^M(\cdot\mid s,a).
\]
Thus the reward and next-state concentration arguments require no independence
between $R_{x,j}$ and $Y_{x,j}$ within a visit.

\emph{Posterior calibration.}
Since $\mathcal C_k$ is determined by $\mathcal F_k$,
\[
\mathbb P(M_k\notin\mathcal C_k\mid\mathcal F_k)
=\int\mathbf1\{M\notin\mathcal C_k\}\,\rho_k(dM)
=\mathbb P(M^\star\notin\mathcal C_k\mid\mathcal F_k)
\qquad\text{a.s.}
\]
Moreover,
\[
\mathbb P(M^\star\notin\mathcal C_k
\text{ for some }k\le K)
=\int
\mathbb P_M(M\notin\mathcal C_k
\text{ for some }k\le K)\,\rho(dM)
\le\frac1{4K},
\]
so
\[
\begin{aligned}
\sum_{k=1}^K\mathbb P(\mathcal G_k^c)
&\le\sum_{k=1}^K\{
\mathbb P(M^\star\notin\mathcal C_k)
+\mathbb P(M_k\notin\mathcal C_k)\}
=2\sum_{k=1}^K\mathbb P(M^\star\notin\mathcal C_k)\\
&\le2K\,\mathbb P(M^\star\notin\mathcal C_k
\text{ for some }k\le K)\le\frac12\le1.
\end{aligned}
\]
This proves the posterior claim without posterior contraction or factorization.
\end{proof}

\subsection{Empirical reference and variance closure}
\label{app:reference-bridge}

For a distribution $p$ on $\mathcal S$, write
$\operatorname{Var}_p(v):=\sum_i p(i)(v(i)-p^\top v)^2$.
At a visited row, let
\[
N_{k,h}:=N_k\bigl((h,S_{k,h},A_{k,h})\bigr),\qquad
\operatorname{Var}_{k,h}(v):=
\operatorname{Var}_{P_h^{M^\star}(\cdot\mid S_{k,h},A_{k,h})}(v).
\]
For $N_{k,h}>0$, take the common empirical reference
$\widehat p=\widehat p_{(h,S_{k,h},A_{k,h}),N_{k,h}}$.
Evaluating every kernel at $(S_{k,h},A_{k,h})$ gives
\[
e^P_{k,h}
=(P_h^{M_k}-\widehat p)V^\star_{h+1,M_k}
+(\widehat p-P_h^{M^\star})V^\star_{h+1,M^\star}
+(\widehat p-P_h^{M^\star})
(V^\star_{h+1,M_k}-V^\star_{h+1,M^\star}).
\]
The first two terms have candidate-aligned directions. The last term requires
the following uniform bound.

\begin{restatement}{Lemma}{lem:projection-variance-transfer}{Uniform error and variance bounds}
If $M\in\mathcal C_k$, then for every row $x=(h,s,a)$, prefix
$1\le n\le N_k(x)$, and $v\in[-H,H]^S$, setting
$p=P_h^M(\cdot\mid s,a)$ and $\widehat p=\widehat p_{x,n}$ gives
\[
|(\widehat p-p)^\top v|
\le
C\sqrt{\frac{S\beta_K}{n}\operatorname{Var}_p(v)}
+\frac{CSH\beta_K}{n}.
\]
For $n\ge2$, also
\[
\frac{n}{n-1}\operatorname{Var}_{\widehat p}(v)
\le
C\left(\operatorname{Var}_p(v)+\frac{SH^2\beta_K}{n}\right).
\]
\end{restatement}

\begin{proof}
Since $(\widehat p-p)^\top\mathbf1=0$,
\[
| (\widehat p-p)^\top v |
\le\sum_i|\widehat p(i)-p(i)|\,|v(i)-p^\top v|.
\]
Since $M\in\mathcal C_k$, the coordinate condition in its definition gives
$|\widehat p(i)-p(i)|\le\sqrt{2p(i)\beta_K/n}+2\beta_K/(3n)$ for every $i$.
Hence
\[
| (\widehat p-p)^\top v |
\le\sqrt{\frac{2\beta_K}{n}}
\sum_i\sqrt{p(i)}\,|v(i)-p^\top v|
+\frac{2\beta_K}{3n}\sum_i|v(i)-p^\top v|.
\]
By Cauchy--Schwarz for the first sum and
$|v(i)-p^\top v|\le2H$ for each term of the second sum, the right-hand side
is at most
\[
\sqrt{\frac{2\beta_K}{n}}
\sqrt{S\sum_i p(i)(v(i)-p^\top v)^2}
+\frac{4SH\beta_K}{3n}
\le C\sqrt{\frac{S\beta_K\operatorname{Var}_p(v)}{n}}
+\frac{CSH\beta_K}{n}.
\]
For the fixed row, write $Y_i=Y_{x,i}$. For $n\ge2$,
\[
\frac1{n(n-1)}\sum_{i<j}(v(Y_i)-v(Y_j))^2
=\frac{n}{n-1}\operatorname{Var}_{\widehat p}(v)
\le2\operatorname{Var}_{\widehat p}(v)
\le2\widehat p^\top(v-(p^\top v)\mathbf1)^2.
\]
The last inequality holds because the empirical mean $\widehat p^\top v$
minimizes the empirical squared deviation over all constant centers.
Adding and subtracting the true expectation of the same centered square gives
\[
\widehat p^\top\bigl(v-(p^\top v)\mathbf1\bigr)^2
=\operatorname{Var}_p(v)
+(\widehat p-p)^\top\bigl(v-(p^\top v)\mathbf1\bigr)^2.
\]
Thus the remaining term is the difference between the empirical and true
expectations of this centered square. By the coordinate confidence condition,
\[
\left|(\widehat p-p)^\top
\bigl(v-(p^\top v)\mathbf1\bigr)^2\right|
\le
\sqrt{\frac{2\beta_K}{n}}
\sum_i\sqrt{p(i)}\bigl(v(i)-p^\top v\bigr)^2
+\frac{2\beta_K}{3n}
\sum_i\bigl(v(i)-p^\top v\bigr)^2.
\]
By Cauchy--Schwarz, followed by
$(v(i)-p^\top v)^4\le4H^2(v(i)-p^\top v)^2$,
\[
\sum_i\sqrt{p(i)}(v(i)-p^\top v)^2
\le\sqrt{S\sum_i p(i)(v(i)-p^\top v)^4}
\le2H\sqrt{S\operatorname{Var}_p(v)}.
\]
Using this bound and $\sum_i(v(i)-p^\top v)^2\le4SH^2$, then applying
Young's inequality, yields
\[
\left|(\widehat p-p)^\top\bigl(v-(p^\top v)\mathbf1\bigr)^2\right|
\le2H\sqrt{\frac{2S\beta_K\operatorname{Var}_p(v)}{n}}
+\frac{8H^2S\beta_K}{3n}
\le C\operatorname{Var}_p(v)+\frac{CH^2S\beta_K}{n}.
\]
Therefore
\[
\frac{n}{n-1}\operatorname{Var}_{\widehat p}(v)
\le C\operatorname{Var}_p(v)+\frac{CH^2S\beta_K}{n}.
\]
\end{proof}

\begin{restatement}{Lemma}{lem:local-reference-bridge}{Local reference bound}
On $\mathcal G_k$, every realized row satisfies
\begin{align*}
|e^P_{k,h}|&\le
C\sqrt{\frac{\beta_K}{N_{k,h}\vee1}
\left(\Var_{k,h}(V^\star_{h+1,M_k})
+\Var_{k,h}(V^\star_{h+1,M^\star})\right)}\\
&\quad+C\sqrt{\frac{S\beta_K}{N_{k,h}\vee1}
\Var_{k,h}
(V^\star_{h+1,M_k}-V^\star_{h+1,M^\star})}
+\frac{CSH\beta_K}{N_{k,h}\vee1},\\
|e^r_{k,h}|
&\le
C\sqrt{\frac{\beta_K}{N_{k,h}\vee1}}.
\end{align*}
\end{restatement}

\begin{proof}
First consider $N_{k,h}\ge2$ on $\mathcal G_k$. At
$x=(h,S_{k,h},A_{k,h})$, let $\widehat p=\widehat p_{x,N_{k,h}}$.
Using $M_k\in\mathcal C_k$ for
the sample-aligned condition and $M=M^\star\in\mathcal C_k$ in
Lemma~\ref{lem:projection-variance-transfer} gives
\[
\begin{aligned}
\left|(P_h^{M_k}-\widehat p)V^\star_{h+1,M_k}\right|
&\le C\sqrt{\frac{\beta_K}{N_{k,h}}
\widehat{\operatorname{Var}}_{x,N_{k,h}}(V^\star_{h+1,M_k})}
+\frac{CH\beta_K}{N_{k,h}},\\
\widehat{\operatorname{Var}}_{x,N_{k,h}}(V^\star_{h+1,M_k})
&\le C\operatorname{Var}_{k,h}(V^\star_{h+1,M_k})
+\frac{CH^2S\beta_K}{N_{k,h}},\\
\left|(P_h^{M_k}-\widehat p)V^\star_{h+1,M_k}\right|
&\le C\sqrt{\frac{\beta_K}{N_{k,h}}}
\sqrt{\operatorname{Var}_{k,h}(V^\star_{h+1,M_k})}
+\frac{CSH\beta_K}{N_{k,h}}.
\end{aligned}
\]
The same calculation with $M^\star$ gives
\[
\begin{aligned}
\left|(\widehat p-P_h^{M^\star})V^\star_{h+1,M^\star}\right|
&\le C\sqrt{\frac{\beta_K}{N_{k,h}}}
\sqrt{\operatorname{Var}_{k,h}(V^\star_{h+1,M^\star})}
+\frac{CSH\beta_K}{N_{k,h}}.
\end{aligned}
\]
For the mismatch direction,
\[
\left|(\widehat p-P_h^{M^\star})
(V^\star_{h+1,M_k}-V^\star_{h+1,M^\star})\right|
\le C\sqrt{\frac{S\beta_K}{N_{k,h}}}
\sqrt{\operatorname{Var}_{k,h}
(V^\star_{h+1,M_k}-V^\star_{h+1,M^\star})}
+\frac{CSH\beta_K}{N_{k,h}}.
\]
Adding the three terms in the empirical-reference decomposition proves the
transition bound in Lemma~\ref{lem:local-reference-bridge}. For the reward error,
\[
|e^r_{k,h}|
\le\left|r_h^{M_k}-\frac1{N_{k,h}}
\sum_{j=1}^{N_{k,h}}R_{x,j}\right|
+\left|r_h^{M^\star}-\frac1{N_{k,h}}
\sum_{j=1}^{N_{k,h}}R_{x,j}\right|
\le C\sqrt{\frac{\beta_K}{N_{k,h}}}.
\]
If $N_{k,h}\le1$,
\[
|e^P_{k,h}|\le H
\le\frac{CSH\beta_K}{N_{k,h}\vee1},
\qquad
|e^r_{k,h}|\le1
\le C\sqrt{\frac{\beta_K}{N_{k,h}\vee1}}.
\]
Thus both bounds hold for every visit.
\end{proof}

For the next two results, fix $(\mathcal F_k,M^\star,M_k)$ and take
expectations only over the trajectory generated by $\pi_k$ in $M^\star$.

\begin{restatement}{Lemma}{lem:bellman-square-telescope}{Bellman square telescope}
Let $v_h:\mathcal S\to\mathbb R$ be fixed under this conditioning, with
$v_{H+1}=0$. Then
\[
\mathbb E\sum_{h=1}^H\operatorname{Var}_{k,h}(v_{h+1})
=-v_1(s_1)^2
+\mathbb E\sum_{h=1}^H\Bigl[v_h(S_{k,h})^2-
\bigl(P_h^{M^\star}v_{h+1}(S_{k,h},A_{k,h})\bigr)^2
\Bigr].
\]
Consequently, if $|v_h|\le H$, then
\[
\mathbb E\sum_{h=1}^H\Var_{k,h}(v_{h+1})
\le 2H\,\mathbb E\sum_{h=1}^H
\Bigl|v_h(S_{k,h})-P_h^{M^\star}v_{h+1}(S_{k,h},A_{k,h})\Bigr|.
\]
\end{restatement}

\begin{proof}
Expanding each conditional variance and using the transition law gives
\[
\begin{aligned}
\mathbb E\sum_{h=1}^H\operatorname{Var}_{k,h}(v_{h+1})
&=\mathbb E\sum_{h=1}^H
\left[v_{h+1}(S_{k,h+1})^2-
\bigl(P_h^{M^\star}v_{h+1}(S_{k,h},A_{k,h})\bigr)^2\right]\\
&=-v_1(s_1)^2+\mathbb E\sum_{h=1}^H
\left[v_h(S_{k,h})^2-
\bigl(P_h^{M^\star}v_{h+1}(S_{k,h},A_{k,h})\bigr)^2\right].
\end{aligned}
\]
For the second equality, shifting the index in the first squared term gives
\[
\sum_{h=1}^H v_{h+1}(S_{k,h+1})^2
=\sum_{h=1}^H v_h(S_{k,h})^2-v_1(s_1)^2
+\underbrace{v_{H+1}(S_{k,H+1})^2}_{=0}.
\]
The terms at stages $2,\ldots,H$ occur in both sums; the initial term is
removed and the terminal term vanishes because $S_{k,1}=s_1$ and $v_{H+1}=0$.
If $|v_h|\le H$, then
$|P_h^{M^\star}v_{h+1}|\le H$, so each squared difference is at most
$2H|v_h-P_h^{M^\star}v_{h+1}|$. Dropping $-v_1(s_1)^2\le0$ gives the
claimed inequality.
\end{proof}

\begin{restatement}{Proposition}{prop:three-variance-bounds}{Three variance bounds}
For every episode $k$, with all sums below taken over $h=1,\ldots,H$,
\[
\begin{aligned}
\mathbb E\sum_h\operatorname{Var}_{k,h}(V^\star_{h+1,M_k})
&\le2H^2+2H\,\mathbb E\sum_h|e^P_{k,h}|,\\
\mathbb E\sum_h\operatorname{Var}_{k,h}(V^\star_{h+1,M^\star})
&\le2H^2+2H\Delta_k,\\
\mathbb E\sum_h\operatorname{Var}_{k,h}
(V^\star_{h+1,M_k}-V^\star_{h+1,M^\star})
&\le2H\left[
\mathbb E\sum_h(|e^P_{k,h}|+|e^r_{k,h}|)+\Delta_k
\right].
\end{aligned}
\]
\end{restatement}

\begin{proof}
All three variances are taken under the true transition $P_h^{M^\star}$,
including those involving the sampled value. Below, reward and transition
expressions are evaluated at $(S_{k,h},A_{k,h})$, and current-stage values
at $S_{k,h}$. Lemma~\ref{lem:bellman-square-telescope} bounds each cumulative
variance by $2H$ times the expected sum of the corresponding absolute
one-step residuals, as we now spell out.

\emph{Sampled value.} Expanding the variance gives
\[
\operatorname{Var}_{k,h}(V^\star_{h+1,M_k})
=P_h^{M^\star}\bigl[(V^\star_{h+1,M_k})^2\bigr]
-\bigl(P_h^{M^\star}V^\star_{h+1,M_k}\bigr)^2.
\]
Because $\pi_k$ is optimal in $M_k$, its Bellman equation and the definition
of $e^P_{k,h}$ yield
\[
V^\star_{h,M_k}-P_h^{M^\star}V^\star_{h+1,M_k}
=r_h^{M_k}+e^P_{k,h}.
\]
Apply Lemma~\ref{lem:bellman-square-telescope} with $v_h=V^\star_{h,M_k}$,
which lies in $[0,H]$ and vanishes at $H+1$:
\[
\begin{aligned}
\mathbb E\sum_h\operatorname{Var}_{k,h}(V^\star_{h+1,M_k})
&\le2H\,\mathbb E\sum_h
\bigl|V^\star_{h,M_k}-P_h^{M^\star}V^\star_{h+1,M_k}\bigr|\\
&=2H\,\mathbb E\sum_h|r_h^{M_k}+e^P_{k,h}|
\le2H^2+2H\,\mathbb E\sum_h|e^P_{k,h}|,
\end{aligned}
\]
where the last inequality uses $0\le r_h^{M_k}\le1$.

\emph{True optimal value.} Here
\[
\operatorname{Var}_{k,h}(V^\star_{h+1,M^\star})
=P_h^{M^\star}\bigl[(V^\star_{h+1,M^\star})^2\bigr]
-\bigl(P_h^{M^\star}V^\star_{h+1,M^\star}\bigr)^2.
\]
The action $A_{k,h}$ need not be optimal in $M^\star$. Bellman optimality
and telescoping in expectation give
\[
\begin{aligned}
&V^\star_{h,M^\star}(S_{k,h})-r_h^{M^\star}(S_{k,h},A_{k,h})
-P_h^{M^\star}V^\star_{h+1,M^\star}(S_{k,h},A_{k,h})\ge0,\\
&\mathbb E\sum_{h=1}^H\left[
V^\star_{h,M^\star}(S_{k,h})-r_h^{M^\star}(S_{k,h},A_{k,h})
-P_h^{M^\star}V^\star_{h+1,M^\star}(S_{k,h},A_{k,h})
\right]
=V^\star_{1,M^\star}(s_1)
-V^{\pi_k}_{1,M^\star}(s_1)
=\Delta_k.
\end{aligned}
\]
In particular, the residual $V^\star_{h,M^\star}-P_h^{M^\star}V^\star_{h+1,M^\star}$
is the nonnegative reward plus the nonnegative Bellman gap above. Applying
Lemma~\ref{lem:bellman-square-telescope} with $v_h=V^\star_{h,M^\star}$ gives
\[
\begin{aligned}
\mathbb E\sum_h\operatorname{Var}_{k,h}(V^\star_{h+1,M^\star})
&\le2H\,\mathbb E\sum_h
\bigl(V^\star_{h,M^\star}-P_h^{M^\star}V^\star_{h+1,M^\star}\bigr)\\
&=2H\,\mathbb E\sum_h r_h^{M^\star}+2H\Delta_k
\le2H^2+2H\Delta_k.
\end{aligned}
\]

\emph{Value difference.} Expanding this variance as a single squared
difference gives
\[
\begin{aligned}
\operatorname{Var}_{k,h}(V^\star_{h+1,M_k}-V^\star_{h+1,M^\star})
&=P_h^{M^\star}\bigl[(V^\star_{h+1,M_k}-V^\star_{h+1,M^\star})^2\bigr]\\
&\quad-\bigl[P_h^{M^\star}(V^\star_{h+1,M_k}-V^\star_{h+1,M^\star})\bigr]^2.
\end{aligned}
\]
Subtracting the two value residuals above yields
\[
\begin{aligned}
&\bigl(V^\star_{h,M_k}-V^\star_{h,M^\star}\bigr)(S_{k,h})
-P_h^{M^\star}\bigl(V^\star_{h+1,M_k}-V^\star_{h+1,M^\star}\bigr)
(S_{k,h},A_{k,h})\\
&\quad=e^r_{k,h}+e^P_{k,h}
-\left[V^\star_{h,M^\star}-r_h^{M^\star}
-P_h^{M^\star}V^\star_{h+1,M^\star}\right].
\end{aligned}
\]
Apply Lemma~\ref{lem:bellman-square-telescope} with
$v_h=V^\star_{h,M_k}-V^\star_{h,M^\star}$, which lies in $[-H,H]$ and
vanishes at $H+1$. The triangle inequality and the expected Bellman-gap
sum $\Delta_k$ then give
\[
\begin{aligned}
\mathbb E\sum_h\operatorname{Var}_{k,h}
(V^\star_{h+1,M_k}-V^\star_{h+1,M^\star})
&\le2H\,\mathbb E\sum_h
\Bigl|\bigl(V^\star_{h,M_k}-V^\star_{h,M^\star}\bigr)\\
&\hspace{4.5em}-P_h^{M^\star}
\bigl(V^\star_{h+1,M_k}-V^\star_{h+1,M^\star}\bigr)\Bigr|\\
&\le2H\left[
\mathbb E\sum_h(|e^r_{k,h}|+|e^P_{k,h}|)+\Delta_k\right].
\end{aligned}
\]
This proves all three inequalities in
Proposition~\ref{prop:three-variance-bounds}.
\end{proof}

\subsection{Visit-Count Sums}
\label{app:visit-count-sums}

\begin{lemma}[Visit-count sums]
\label{lem:visit-count-sums}
The visit counts satisfy, pathwise,
\begin{equation}
\label{eq:visit-count-bounds}
\sum_{k,h}\tfrac{1}{N_{k,h}\vee1}\le SAHL_K,\quad
\sum_{k,h}\tfrac{1}{\sqrt{N_{k,h}\vee1}}
\le2H\sqrt{SAK}.
\end{equation}
\end{lemma}

\begin{proof}
Each row $x=(h,s,a)$ can be visited at most once per episode because its
stage $h$ is fixed. Thus, if a row is visited $m$ times, then $m\le K$
and the pre-episode count on its $j$th visit is exactly $j-1$.
Unvisited rows contribute zero; for $m=1$, both row sums equal one.
For $m\ge2$, the first two denominators equal one. Since $t^{-1}$ and
$t^{-1/2}$ decrease, each later summand is at most the integral over the
preceding unit interval. Hence
\[
\begin{aligned}
\sum_{j=1}^{m}\frac1{(j-1)\vee1}
&=1+\sum_{j=1}^{m-1}\frac1j
\le2+\int_1^{m-1}\frac{dt}{t}
=2+\log(m-1)\le2+\log m,\\
\sum_{j=1}^{m}\frac1{\sqrt{(j-1)\vee1}}
&=1+\sum_{j=1}^{m-1}\frac1{\sqrt j}
\le2+\int_1^{m-1}\frac{dt}{\sqrt t}
=2\sqrt{m-1}\le2\sqrt m.
\end{aligned}
\]
Write $N_{K+1}(x)$ for the total visits to row $x$ in the $K$ episodes.
There are $SAH$ rows, each has at most $K$ visits, and exactly $KH$ visits
occur in total. Grouping the first sum by row therefore gives
\[
\sum_{k,h}\frac1{N_{k,h}\vee1}
\le\sum_{x:N_{K+1}(x)>0}\bigl(2+\log N_{K+1}(x)\bigr)
\le SAH(2+\log K)=SAHL_K.
\]
For the second sum, apply the row bound and then Cauchy--Schwarz to the
$SAH$ numbers $\sqrt{N_{K+1}(x)}$:
\[
\sum_{k,h}\frac1{\sqrt{N_{k,h}\vee1}}
\le2\sum_x\sqrt{N_{K+1}(x)}
\le2\sqrt{\Bigl(\sum_x1\Bigr)\Bigl(\sum_xN_{K+1}(x)\Bigr)}
=2\sqrt{SAH\cdot KH}=2H\sqrt{SAK}.
\]
Both inequalities hold pathwise for arbitrary adaptive policies.
\end{proof}

\subsection{Completing the Regret Bound}
\label{app:regret-closure}

\begin{restatement}{Theorem}{thm:arbitrary-correlated-prior}{Minimax tabular regret under an arbitrary prior}
For every prior $\rho$ and every $K\ge 1$, exact vanilla PSRL satisfies
\[
\operatorname{BReg}_{\rho}(K)
\le
C\sqrt{SAH^3K\beta_KL_K}
+CS^2AH^2\beta_KL_K.
\]
\end{restatement}

\begin{proof}
Define
\[
W_P:=\sum_{k,h}\mathbb E[\mathbf1_{\mathcal G_k}|e^P_{k,h}|],
\qquad
W_r:=\sum_{k,h}\mathbb E[\mathbf1_{\mathcal G_k}|e^r_{k,h}|].
\]
The expectations here include the prior, posterior draws, and trajectories.
The identities in Lemma~\ref{lem:posterior-simulation-identities} give
\[
\begin{aligned}
\operatorname{BReg}_{\rho}(K)
&=\sum_{k=1}^K\mathbb E\!\left[
V^{\pi_k}_{1,M_k}(s_1)-V^{\pi_k}_{1,M^\star}(s_1)\right]\\
&=\sum_{k=1}^K\mathbb E\!\left[
\mathbf1_{\mathcal G_k}
\bigl(V^{\pi_k}_{1,M_k}(s_1)-V^{\pi_k}_{1,M^\star}(s_1)\bigr)\right]
+\sum_{k=1}^K\mathbb E\!\left[
\mathbf1_{\mathcal G_k^c}
\bigl(V^{\pi_k}_{1,M_k}(s_1)-V^{\pi_k}_{1,M^\star}(s_1)\bigr)\right].
\end{aligned}
\]
Because $\mathcal G_k$ is fixed before the episode trajectory,
\[
\begin{aligned}
&\mathbb E\!\left[\mathbf1_{\mathcal G_k}
\bigl(V^{\pi_k}_{1,M_k}(s_1)-V^{\pi_k}_{1,M^\star}(s_1)\bigr)\right]
\\&\quad=\mathbb E\!\left[\mathbf1_{\mathcal G_k}
\mathbb E\!\left[\sum_{h=1}^H(e^P_{k,h}+e^r_{k,h})
\,\middle|\,\mathcal F_k,M^\star,M_k\right]\right]
=\mathbb E\!\left[\mathbf1_{\mathcal G_k}
\sum_{h=1}^H(e^P_{k,h}+e^r_{k,h})\right].
\end{aligned}
\]
Thus
\[
\operatorname{BReg}_{\rho}(K)
\le W_P+W_r+H\sum_{k=1}^K\mathbb P(\mathcal G_k^c)
\le W_P+W_r+H,
\]
and
\[
W_r\le C\sqrt{\beta_K}\sum_{k,h}
\mathbb E\!\left[\frac1{\sqrt{N_{k,h}\vee1}}\right]
\le CH\sqrt{SAK\beta_K}.
\]
For the sample-aligned term,
\[
\sum_{k,h}\mathbb E\!\left[
\mathbf1_{\mathcal G_k}
\sqrt{\frac{\beta_K}{N_{k,h}\vee1}}
\sqrt{\operatorname{Var}_{k,h}(V^\star_{h+1,M_k})}
\right]
\le\sqrt{SAH\beta_KL_K}
\sqrt{\sum_{k,h}\mathbb E\!\left[
\mathbf1_{\mathcal G_k}
\operatorname{Var}_{k,h}(V^\star_{h+1,M_k})\right]}.
\]
For the truth-aligned term,
\[
\sum_{k,h}\mathbb E\!\left[
\mathbf1_{\mathcal G_k}
\sqrt{\frac{\beta_K}{N_{k,h}\vee1}}
\sqrt{\operatorname{Var}_{k,h}(V^\star_{h+1,M^\star})}\right]
\le\sqrt{SAH\beta_KL_K}
\sqrt{\sum_{k,h}\mathbb E\!\left[
\mathbf1_{\mathcal G_k}
\operatorname{Var}_{k,h}(V^\star_{h+1,M^\star})\right]}.
\]
Hence the two aligned contributions are at most
\[
C\sqrt{SAH\beta_KL_K}\,\Bigg\{
\sum_{k,h}\mathbb E\!\left[
\mathbf1_{\mathcal G_k}\operatorname{Var}_{k,h}(V^\star_{h+1,M_k})\right]
+\sum_{k,h}\mathbb E\!\left[
\mathbf1_{\mathcal G_k}\operatorname{Var}_{k,h}(V^\star_{h+1,M^\star})\right]
\Bigg\}^{1/2}.
\]
The mismatch contribution is at most
\[
\begin{aligned}
&C\sqrt{S^2AH\beta_KL_K}
\Bigg\{\sum_{k,h}\mathbb E\!\left[
\mathbf1_{\mathcal G_k}\operatorname{Var}_{k,h}\!\left(
V^\star_{h+1,M_k}-V^\star_{h+1,M^\star}\right)
\right]\Bigg\}^{1/2},
\end{aligned}
\]
and
\[
CSH\beta_K\sum_{k,h}\frac1{N_{k,h}\vee1}
\le CS^2AH^2\beta_KL_K.
\]
By Proposition~\ref{prop:three-variance-bounds},
\[
\begin{aligned}
&\sum_{k,h}\mathbb E\!\left[\mathbf1_{\mathcal G_k}\{
\operatorname{Var}_{k,h}(V^\star_{h+1,M_k})
+\operatorname{Var}_{k,h}(V^\star_{h+1,M^\star})\}\right]
\le4H^2K+2H\{W_P+\operatorname{BReg}_{\rho}(K)\},\\
&\sum_{k,h}\mathbb E\!\left[
\mathbf1_{\mathcal G_k}\operatorname{Var}_{k,h}
(V^\star_{h+1,M_k}-V^\star_{h+1,M^\star})\right]
\le2H\{W_P+W_r+\operatorname{BReg}_{\rho}(K)\}.
\end{aligned}
\]
Therefore
\[
W_P
\le C\sqrt{SAH^3K\beta_KL_K}+CS^2AH^2\beta_KL_K
+C\sqrt{S^2AH^2\beta_KL_K
\{W_P+W_r+\operatorname{BReg}_{\rho}(K)\}}.
\]
Using the regret bound $\operatorname{BReg}_\rho(K)\le W_P+W_r+H$ above and Young's inequality,
\[
\begin{aligned}
&C\sqrt{S^2AH^2\beta_KL_K
\{W_P+W_r+\operatorname{BReg}_{\rho}(K)\}}
\le C\sqrt{S^2AH^2\beta_KL_K(2W_P+2W_r+H)}\\
&\quad\le\frac12W_P+CW_r+CS^2AH^2\beta_KL_K+CH.
\end{aligned}
\]
Hence
\[
\begin{aligned}
W_P
&\le C\sqrt{SAH^3K\beta_KL_K}+CW_r
+CS^2AH^2\beta_KL_K+CH,\\
\operatorname{BReg}_{\rho}(K)
&\le C\sqrt{SAH^3K\beta_KL_K}+CS^2AH^2\beta_KL_K
+CH\sqrt{SAK\beta_K}+CH.
\end{aligned}
\]
Since $S,A,H,K\ge1$ and $\beta_K,L_K\ge1$, the last two terms are
absorbed by the leading term after increasing $C$. This gives the stated
bound, whose right-hand side may also be capped by $HK$.
In particular, if $K\ge CS^3AH\beta_KL_K$,
\[
S^2AH^2\beta_KL_K
\le C\sqrt{SAH^3K\beta_KL_K},
\qquad
H\sqrt{SAK\beta_K}+H
\le C\sqrt{SAH^3K\beta_KL_K},
\]
so the leading term suffices in this regime.
\end{proof}

\subsection{Fixed versus Uniform Control}

\label{app:fixed-uniform-projection}

\begin{proposition}[Fixed versus uniform control]
\label{prop:fixed-versus-uniform}
Let $p$ be uniform on $\mathcal S$, with $S\ge2$, and let $\widehat p$ be the
empirical distribution of $n\ge1$ independent draws from $p$. For every fixed
$v\in[0,H]^S$,
\[
\mathbb E|(\widehat p-p)^\top v|
\le\sqrt{\frac{\operatorname{Var}_p(v)}{n}}
\le\frac{H}{2\sqrt n}.
\]
In contrast,
\[
\mathbb E\!\left[\sup_{v\in[0,H]^S}
|(\widehat p-p)^\top v|\right]
=\frac H2\mathbb E\|\widehat p-p\|_1
=\Theta\!\left(H\min\left\{1,\sqrt{\frac Sn}\right\}\right).
\]
\end{proposition}

\begin{proof}
\[
\mathbb E|(\widehat p-p)^\top v|
\le\sqrt{\mathbb E[((\widehat p-p)^\top v)^2]}
=\sqrt{\frac{\operatorname{Var}_p(v)}{n}}
\le\frac{H}{2\sqrt n}.
\]
Since $(\widehat p-p)^\top\mathbf1=0$,
\[
\sup_{v\in[0,H]^S}|(\widehat p-p)^\top v|
=\frac H2\|\widehat p-p\|_1.
\]
The upper bound follows from
\[
\mathbb E\|\widehat p-p\|_1
\le\min\left\{2,\sqrt{S\mathbb E\|\widehat p-p\|_2^2}\right\}
\le\min\left\{2,\sqrt{\frac Sn}\right\}.
\]
If $n<S$,
\[
\mathbb E\|\widehat p-p\|_1
\ge\sum_{i=1}^S\frac1S\,
\mathbb P\!\left(\widehat p(i)=0\right)
=\left(1-\frac1S\right)^n
\ge\frac14.
\]
If $n\ge S$, the binomial second and fourth moments give
\[
\mathbb E\left(\widehat p(i)-\frac1S\right)^2
=\frac{(1/S)(1-1/S)}{n},
\qquad
\mathbb E\left(\widehat p(i)-\frac1S\right)^4
\le\frac{C}{n^2S^2},
\]
so Paley--Zygmund yields
\[
\mathbb E\left|\widehat p(i)-\frac1S\right|
\ge\frac1{C\sqrt{nS}},
\qquad
\mathbb E\|\widehat p-p\|_1
\ge\frac1C\sqrt{\frac Sn}.
\]
Combining the two regimes proves the claim.
\end{proof}

\section{Linear-Mixture MDP Extension}
\label{app:linear-mixture}

This section gives the complete linear-mixture argument. After the setting
and proof roadmap, we state the confidence-width guarantee needed for the
analysis and use it to prove the regret bound. The confidence-set
construction and the proof of that guarantee are given last.

\subsection{Linear-Mixture Setting}

Consider a finite-horizon, time-inhomogeneous MDP with finite state and action
sets, a fixed initial state $s_1$, and known deterministic rewards
$r_h(s,a)\in[0,1]$. At stage $h$, the transition kernel has the linear-mixture
form \citep{ayoub2020,zhou2021}
\[
P_h^\theta(s'\mid s,a)=\langle\phi_h(s,a,s'),\theta_h\rangle,
\qquad
\phi_{h,V}(s,a):=\sum_{s'}\phi_h(s,a,s')V(s').
\]
The feature maps are known. The unknown parameter tuple
$\theta=(\theta_1,\ldots,\theta_H)$ induces valid transition kernels and
satisfies, for every $h,s,a$ and $V\in[0,H]^{\mathcal S}$,
\[
\|\theta_h\|_2\le B,
\qquad
\|\phi_{h,V}(s,a)\|_2\le\|V\|_\infty.
\]
Let $\rho$ be any joint prior on this parameter class and draw
$\theta^\star\sim\rho$. No independence across stages or coordinates is
assumed. Write $V^\pi_{h,\theta}$ and $V^\star_{h,\theta}$ for the value of
policy $\pi$ and the optimal value in model $\theta$, with terminal value zero.

Before episode $k$, let $\cF_k$ contain the preceding observations and
posterior draws, and let $\rho_k$ be the posterior given $\cF_k$. Vanilla PSRL
draws a fresh $\theta_k\sim\rho_k$, computes an optimal policy
$\pi_k=\pi^\star_{\theta_k}$, and executes it for the entire episode.
Sampling and planning are exact; fix a measurable rule for breaking ties.
Thus, conditionally on $\cF_k$, both $\theta^\star$ and $\theta_k$ have law
$\rho_k$, and the fresh draw uses independent algorithmic randomness.
Let $(S_{k,h},A_{k,h})$ denote the resulting trajectory. The Bayesian regret is
\[
\BReg_\rho(K):=
\E\sum_{k=1}^K
\left[V^\star_{1,\theta^\star}(s_1)-V^{\pi_k}_{1,\theta^\star}(s_1)\right].
\]
For a sufficiently large universal constant $c$, define
\[
\Lambda_K:=c\left[1+\log\!\bigl(c d H K(1+B)\bigr)\right].
\]

We prove the following explicit form of Theorem~\ref{thm:linear-mixture}:
for every joint prior $\rho$ and every $K\ge1$,
\[
\BReg_\rho(K)\le C d\sqrt{H^3K}\,\Lambda_K^4,
\]
where $C$ is universal.

Algorithm~\ref{alg:linear-mixture-psrl} specializes
Algorithm~\ref{alg:vanilla-psrl} to linear-mixture MDPs. As in
\citep[Algorithm~2]{li2024}, we sample the model parameters, here allowing
an arbitrary joint prior.

\begin{algorithm}[H]
\caption{PSRL in linear-mixture MDPs}
\label{alg:linear-mixture-psrl}
\begin{algorithmic}[1]
\STATE \textbf{Input:} joint prior $\rho$
\FOR{episode $k=1,\ldots,K$}
    \STATE Sample $\theta_k\sim\mathbb P(\theta^\star\in\cdot\mid\mathcal F_k)$
    \STATE Compute an optimal policy $\pi_k=\pi_{\theta_k}^\star$
    \STATE Execute $\pi_k$ for one episode and update $\mathcal F_{k+1}$
\ENDFOR
\end{algorithmic}
\end{algorithm}

\Needspace{12\baselineskip}
\paragraph{Proof roadmap.}
The main difficulty is to control the transition error evaluated at the
sampled model's own value function. We first state the confidence-width
guarantee needed for this task, use it to prove the regret bound, and then
give its construction and proof.
\begin{enumerate}
\item \textbf{Express next-state values as linear observations.}
At each stage, evaluate the observed next state using the current sampled
continuation value. This value is fixed before the transition, so the
conditional mean of the observation is linear in the true parameter and
its noise variance is the variance of that sampled value under the true
transition law.

\item \textbf{State the required confidence-width guarantee.}
Section~\ref{lin:subsec:observations} states that one construction of
pre-episode confidence sets covers the true parameter and bounds cumulative
prediction disagreement by the accumulated noise variance. The sets use
no information from the fresh posterior draw.

\item \textbf{Convert the width guarantee into regret.}
Section~\ref{lin:subsec:regret-proof} applies posterior calibration to cover
the sampled parameter as well. The confidence width then bounds the
transition errors. A Bellman-square argument controls the remaining
cross-model variance by $2H^2K+2H W_P$, where
$W_P=\sum_{k,h}\mathbb E|e^P_{k,h}|$. Solving the resulting inequality for
$W_P$ gives the regret bound.

\item \textbf{Construct the sets and prove the width guarantee.}
Section~\ref{lin:subsec:width-construction} supplies the deferred argument.
It groups observations by uncertainty, tests candidates using their squared
prediction errors, and establishes coverage, candidate-distance bounds,
and cumulative width in that order.
\end{enumerate}
All confidence sets, scales, and weights are used only in the analysis;
vanilla PSRL continues to sample one model and follow its optimal policy
in each episode.

\Needspace{12\baselineskip}
\subsection{A Confidence-Width Guarantee}
\label{lin:subsec:observations}

We first state the sole confidence-set result used in the regret proof.
Fix an admissible environment $\theta$; the posterior draw $\theta_k$ still
varies between episodes. At stage $h$, evaluate the observed next state
using the sampled model's continuation value and define
\begin{equation}
\begin{aligned}
x_{k,h}
&:=\frac{\phi_{h,V^\star_{h+1,\theta_k}}(S_{k,h},A_{k,h})}{H},
&
y_{k,h}
&:=\frac{V^\star_{h+1,\theta_k}(S_{k,h+1})}{H},\\
\sigma_{k,h}^2(\theta)
&:=\frac{\Var_{P_h^\theta(\cdot\mid S_{k,h},A_{k,h})}
(V^\star_{h+1,\theta_k})}{H^2}.
\end{aligned}
\label{lin:eq:sampled-value-data}
\end{equation}
The feature $x_{k,h}$ is known before the next state arrives, and $y_{k,h}$
is the value of that next state, rather than the reward or a realized
return. A candidate parameter $\vartheta$ predicts $y_{k,h}$ by
$\langle\vartheta,x_{k,h}\rangle$. The noise variance
$\sigma_{k,h}^2(\theta)$ is used only in the analysis and need not be known
to the learner.

Before the next state is observed, the available information is
$\mathcal H_{k,h}:=\sigma(\cF_k,\theta_k,S_{k,1:h},A_{k,1:h})$.
Under the fixed environment $\theta$,
\begin{equation}
\begin{aligned}
\|x_{k,h}\|_2&\le1,\qquad 0\le y_{k,h}\le1,\\
\E[y_{k,h}\mid\mathcal H_{k,h}]&=\langle\theta_h,x_{k,h}\rangle,
&
\E[(y_{k,h}-\langle\theta_h,x_{k,h}\rangle)^2\mid\mathcal H_{k,h}]
&=\sigma_{k,h}^2(\theta).
\end{aligned}
\label{lin:eq:explicit-filtration}
\end{equation}
Although the sampled value changes between episodes, it is fixed before
each next-state observation. Thus the centered observation has conditional
mean zero and absolute value at most one.

\begin{restatement}{Lemma}{lem:linear-sampled-width}{Sampled-value confidence width}
There is a single construction of closed confidence sets
$\mathcal C_{k,h}\subseteq\{\vartheta:\|\vartheta\|_2\le B\}$ based only on
the history before episode $k$, without using $\theta_k$, such that for
every fixed admissible $\theta$, with probability at least $1-(4HK)^{-2}$,
$\theta_h\in\mathcal C_{k,h}$ for all $k,h$ and, simultaneously for all $h$,
\begin{equation}
\sum_{k=1}^K
\Bigl[1\wedge\sup_{\vartheta,\vartheta'\in\mathcal C_{k,h}}
|\langle\vartheta-\vartheta',x_{k,h}\rangle|\Bigr]
\le Cd\Lambda_K^4
\biggl(\sqrt{\sum_{k=1}^K\sigma_{k,h}^2(\theta)}+1\biggr).
\label{lin:eq:confidence-width}
\end{equation}
\end{restatement}

The supremum is the largest disagreement between two candidates'
predictions for the current observation. Once both the true and sampled
parameters lie in the set, it bounds their normalized value-weighted
transition error. Clipping at one suffices because this error between
valid models is at most one. If a set is empty, define its width to be zero.
The sets may use earlier posterior draws in
$\cF_k$, but the same construction must work for every fixed environment;
it cannot depend on the unknown $\theta$.
Section~\ref{lin:subsec:width-construction} gives the construction and proof.
The next section uses only this guarantee.

\subsection{Variance Closure and Regret Bound}
\label{lin:subsec:regret-proof}

Define the value-weighted transition error and its cumulative magnitude by
\[
e^P_{k,h}:=\left\langle\theta_{k,h}-\theta_h^\star,
\phi_{h,V^\star_{h+1,\theta_k}}(S_{k,h},A_{k,h})\right\rangle,
\qquad W_P:=\sum_{k,h}\mathbb E|e^P_{k,h}|.
\]
All transition means and variances below are evaluated at
$(S_{k,h},A_{k,h})$.

\paragraph{Regret reduction and posterior calibration.}
Since rewards are known, only the transition error contributes.
Conditional equality of the posterior
laws replaces the true optimal value by the sampled optimal value in
expectation. Bellman telescoping along the executed policy then gives
\begin{equation}
\BReg_\rho(K)
=\sum_{k=1}^K\E\left[
V^\star_{1,\theta_k}(s_1)-V^{\pi_k}_{1,\theta^\star}(s_1)\right]
=\sum_{k,h}\E e^P_{k,h}\le W_P.
\label{lin:eq:regret-by-error}
\end{equation}
Use the confidence sets supplied by Lemma~\ref{lem:linear-sampled-width}.
For the complete-model confidence set, intersect
$\prod_{h=1}^H\cC_{k,h}$ with the set of parameter tuples that induce valid
transition kernels. This product combines the stage-wise tests; it imposes
no factorization on the prior. Since $\theta_k\mid\cF_k$ and
$\theta^\star\mid\cF_k$ have the same posterior,
\[
\Pp\!\left(\theta_k\notin\prod_h\cC_{k,h}\,\middle|\,\cF_k\right)
=\rho_k\!\left(\left(\prod_h\cC_{k,h}\right)^c\right)
=\Pp\!\left(\theta^\star\notin\prod_h\cC_{k,h}\,\middle|\,\cF_k\right),
\]
so
\begin{equation}
\sum_{k=1}^K
\Pp\!\left(\theta_k\notin\prod_{h=1}^H\cC_{k,h}\right)
=
\sum_{k=1}^K
\Pp\!\left(\theta^\star\notin\prod_{h=1}^H\cC_{k,h}\right)
\le K(4HK)^{-2}.
\label{lin:eq:posterior-calibration}
\end{equation}
Averaging the fixed-model statement over the prior gives
\[
\Pp\!\left(
\theta_h^\star\in\cC_{k,h}
\text{ and \eqref{lin:eq:confidence-width} holds for all }k,h
\right)
\ge1-(4HK)^{-2}.
\]
On the truth's simultaneous event, for each episode in which
$\theta_k\in\prod_h\cC_{k,h}$,
\begin{align}
\frac{|e^P_{k,h}|}{H}
&\le
1\wedge
\sup_{\vartheta,\vartheta'\in\cC_{k,h}}
\left|\left\langle\vartheta-\vartheta',
\frac{\phi_{h,V_{h+1,\theta_k}^{\star}}(S_{k,h},A_{k,h})}{H}
\right\rangle\right|,
\label{lin:eq:error-by-width}
\end{align}
where $|e^P_{k,h}|\le H$ because both parameters induce transition kernels and
$0\le V_{h+1,\theta_k}^{\star}\le H$. Hence
\[
\sum_{k,h}|e^P_{k,h}|
\le
C d\Lambda_K^4
\sum_{h=1}^H
\sqrt{\sum_{k=1}^K
\Var_{P_h^{\theta^\star}}
(V_{h+1,\theta_k}^{\star})}
+C dH^2\Lambda_K^4
+H^2\sum_{k=1}^K
\mathbf1\!\left\{\theta_k\notin\prod_{h=1}^H\cC_{k,h}\right\}.
\]
On failure of the truth's simultaneous event,
$\sum_{k,h}|e^P_{k,h}|\le H^2K$. Therefore
\begin{align}
W_P
&\le
C d\Lambda_K^4
\E\sum_{h=1}^H
\sqrt{\sum_{k=1}^K
\Var_{P_h^{\theta^\star}}(V_{h+1,\theta_k}^{\star})}
+C dH^2\Lambda_K^4
+2H^2K(4HK)^{-2}\notag\\
&\le
C d\Lambda_K^4
\left[H\sum_{k,h}\E
\Var_{P_h^{\theta^\star}}(V_{h+1,\theta_k}^{\star})\right]^{1/2}
+C dH^2\Lambda_K^4+1.
\label{lin:eq:cumulative-error}
\end{align}
\paragraph{Bellman variance closure.}
It remains to control the variance in \eqref{lin:eq:cumulative-error}.
The following bound uses only the sampled-model Bellman equation and the
square telescope.

\begin{restatement}{Proposition}{prop:linear-cross-model-variance}{Cross-model variance closure}
Let
$W_P:=\sum_{k,h}\mathbb E|e^P_{k,h}|$. Then
\[
\sum_{k,h}\mathbb E\,
\operatorname{Var}_{P_h^{\theta^\star}}
(V^\star_{h+1,\theta_k})
\le 2H^2K+2H W_P.
\]
\end{restatement}

\begin{proof}
The sampled-model Bellman equation gives
\[
P_h^{\theta^\star}V_{h+1,\theta_k}^{\star}
=V_{h,\theta_k}^{\star}-r_h-e^P_{k,h},
\]
where the stage-$h$ values and reward are evaluated at the current state
and action. Expanding each conditional variance and telescoping the squared
values gives
\begin{equation}
\begin{aligned}
\sum_{k,h}\E\Var_{P_h^{\theta^\star}}
(V^\star_{h+1,\theta_k})
&=-\sum_k\E\left[V^\star_{1,\theta_k}(s_1)^2\right]
+\sum_{k,h}\E\left[
V^\star_{h,\theta_k}(S_{k,h})^2
-\left(P_h^{\theta^\star}V^\star_{h+1,\theta_k}\right)^2\right]\\
&\le2H\sum_{k,h}\E|r_h+e^P_{k,h}|
\le2H^2K+2H W_P.
\end{aligned}
\label{lin:eq:variance-closure}
\end{equation}
To obtain the identity, condition on $(\cF_k,\theta^\star,\theta_k)$, expand
each conditional variance, and telescope using $V^\star_{H+1,\theta_k}=0$.
The first inequality uses $a^2-b^2\le2H|a-b|$ for $a,b\in[0,H]$.
\end{proof}

\begin{restatement}{Theorem}{thm:linear-mixture}{Linear-mixture MDPs}
For every joint prior on the admissible parameter tuples and every
$K\ge1$, exact vanilla PSRL satisfies the explicit bound
\[
\operatorname{BReg}_{\rho}(K)
\le C d\sqrt{H^3K}\,\Lambda_K^4
=\widetilde O\!\left(d\sqrt{H^3K}\right).
\]
\end{restatement}

\begin{proof}
By Proposition~\ref{prop:linear-cross-model-variance},
$\sum_{k,h}\E\Var_{P_h^{\theta^\star}}(V^\star_{h+1,\theta_k})\le2H^2K+2HW_P$.
Substituting into \eqref{lin:eq:cumulative-error},
\[
\begin{aligned}
W_P
&\le C d\sqrt{H^3K}\,\Lambda_K^4
+C dH \Lambda_K^4
\sqrt{W_P}
+C dH^2\Lambda_K^4+1,\\
C dH \Lambda_K^4
\sqrt{W_P}
&\le\frac12W_P
+C d^2H^2\Lambda_K^8.
\end{aligned}
\]
Thus
\[
W_P
\le C d\sqrt{H^3K}\,\Lambda_K^4
+C d^2H^2\Lambda_K^8.
\]
If $K\ge d^2H\Lambda_K^8$, the second term is absorbed by the
first. If $K<d^2H\Lambda_K^8$,
\[
\BReg_\rho(K)
\le HK
=H\sqrt K\sqrt K
\le d\sqrt{H^3K}\,\Lambda_K^4.
\]
Therefore, for every $K\ge1$,
\[
\BReg_\rho(K)
\le C d\sqrt{H^3K}\,\Lambda_K^4
=\widetilde O\!\left(d\sqrt{H^3K}\right).
\]
\end{proof}

\subsection{Construction and Proof of the Confidence-Width Guarantee}
\label{lin:subsec:width-construction}

We now prove Lemma~\ref{lem:linear-sampled-width}, which was the only
confidence-set input to Section~\ref{lin:subsec:regret-proof}.

\paragraph{Construction sketch.}
Fix a stage $h$. Before episode $k$, construct the set in five steps:
\begin{enumerate}[leftmargin=*,itemsep=2pt,topsep=3pt]
\item \textbf{Use past prediction data.}
Use $(x_{i,h},y_{i,h})$ from \eqref{lin:eq:sampled-value-data} for $i<k$:
the feature uses the earlier sampled value, and the response evaluates that
same value at the observed next state.
\item \textbf{Group observations by uncertainty.}
Assign observations to geometric scales $\ell$ according to feature uncertainty,
with weights $w_i$ and regularized information matrices $\Sigma_{k,\ell}$.
This lets us control prediction errors separately at each uncertainty level.
\item \textbf{Define a loss and its gradient.}
For a candidate stage parameter $\vartheta\in\mathbb R^d$, define the weighted
quadratic prediction loss $Q_{k,\ell}(\vartheta)$ and the gradient
$g_{k,\ell}(\vartheta)$ of the corresponding regularized loss.
Here $\vartheta$ is a variable in the set definition, not an additional sample.
\item \textbf{Define the confidence set.}
Collect all $\vartheta$ with $\|\vartheta\|_2\le B$ whose normalized gradient
lies within a tolerance determined by $Q_{k,\ell}(\vartheta)$ at every scale.
These vectors form $\mathcal C_{k,h}$.
\item \textbf{Establish coverage and prediction width.}
Concentration shows that the true parameter satisfies the conditions.
A deterministic comparison bounds the prediction differences between vectors
in the set, and scale counts control their cumulative width.
\end{enumerate}
The construction uses only past observations and posterior draws, without the
fresh draw $\theta_k$. All sets and scale assignments are used only in the
analysis; PSRL does not compute them.

\paragraph{Relation to prior work.}
The linear concentration tools build on \citep{zhao2023}.
Lemma~\ref{lin:lem:coverage} applies their Theorem~2.1 directly.
The geometric scales and within-scale weights adapt their Algorithm~1
(SAVE), while the comparison of squared errors with conditional variances,
the fine-scale absorption argument, and the aggregation across scales
follow the ideas in their Appendix~C, including Lemmas~C.4--C.5.
We give the adapted arguments below; these ingredients are not new
concentration results.

Our adaptation tests each candidate using its own squared errors and proves
a deterministic candidate-distance bound
(Lemma~\ref{lin:lem:candidate-comparison}), followed by a cumulative width
bound along PSRL's sampled-value directions
(Proposition~\ref{lin:prop:width-aggregation}).
The additional PSRL argument combines this width bound with standard
posterior calibration under an arbitrary joint prior and our Bellman-square
variance closure (Proposition~\ref{prop:linear-cross-model-variance}), which
controls sampled values under the true transition law without optimism.
Thus the proof uses Zhao et al.'s concentration and layering tools, but
does not invoke a regret theorem for SAVE or UCRL-VTR.

\paragraph{Notation for the construction.}
Fix an admissible environment $\theta$ and a stage $h$, and put
$\delta=(4HK)^{-2}$. Suppress the stage index by writing
$x_k=x_{k,h}$, $y_k=y_{k,h}$, and
$\sigma_k^2=\sigma_{k,h}^2(\theta)$ for the quantities in
\eqref{lin:eq:sampled-value-data}. Here $\theta$ is fixed,
$\theta_k$ is the posterior draw, and $\vartheta$ denotes a candidate
for $\theta_h$. The conditional identities in
\eqref{lin:eq:explicit-filtration} will be used throughout.

\subsubsection{Grouping Observations by Uncertainty}
\label{lin:subsec:construction}

To bound prediction disagreement, we must control both parameter error
and uncertainty in the current feature direction. For a positive definite
matrix $\Sigma$, write $\|z\|_\Sigma:=\sqrt{z^\top\Sigma z}$.
For the information matrix $\Sigma_{k,\ell}$ constructed below,
Cauchy--Schwarz gives
\[
|\langle\vartheta-\theta_h,x_k\rangle|
\le\|\vartheta-\theta_h\|_{\Sigma_{k,\ell}}\,
\|x_k\|_{\Sigma_{k,\ell}^{-1}}.
\]
The first factor measures parameter error in the directions represented
by past observations. The second measures how well those observations
cover the current direction. This geometric uncertainty differs from
$\sigma_k^2$, which measures the randomness of the next-state value even
when the true parameter is known.

The self-normalized bound of \citep[Theorem~2.1]{zhao2023} also depends on
the largest normalized feature size. We therefore use separate information
matrices at geometric scales and rescale each assigned feature to have
normalized size $2^{-\ell}$. Following
\citep[Remark~2.2]{zhao2023}, the weights normalize feature uncertainty
using only observed features. We handle the unknown noise variance
separately through squared prediction errors.
Matrix growth will also limit the number of observations assigned to
each scale, allowing the prediction bounds to be summed.

Retain the geometric scales
\[
1\le\ell\le\ell_{\max}
:=\left\lceil\log_2\max\{4K(1+B),\,256\sqrt{\Lambda_K}\}\right\rceil,
\qquad
\lambda_\ell:=\frac{2^{-2\ell}}{(1+B)^2}.
\]
The regularizer is chosen so that $\sqrt{\lambda_\ell}B\le2^{-\ell}$,
keeping the regularization error within the scale's uncertainty level.
The term $256\sqrt{\Lambda_K}$ ensures that the retained range reaches
fine scales where $2^{\ell-1}\ge128\sqrt{\Lambda_K}$, so the
candidate-distance bound in Lemma~\ref{lin:lem:candidate-comparison}
applies at the preceding scale. The term $4K(1+B)$ ensures
$2^{-\ell_{\max}}\le[4K(1+B)]^{-1}$, which limits the cumulative width
of observations assigned to no scale to $O(\Lambda_K)$, as shown in
Proposition~\ref{lin:prop:width-aggregation}. These constants are convenient
sufficient choices; their exact values are not essential.

For each fixed stage $h$, Algorithm~\ref{alg:analysis-scale-assignment}
groups the episode indices $k$ by scale. The groups and matrices are
constructed separately at each stage; the scale index $\ell$ is unrelated
to $h$.

\begin{algorithm}[H]
\caption{Scale assignment for the analysis}
\label{alg:analysis-scale-assignment}
\textit{Used only in the proof to organize observations generated by PSRL.
This procedure does not affect posterior sampling or action selection.}
\begin{algorithmic}[1]
\REQUIRE Fixed stage $h$, scales $1,\ldots,\ell_{\max}$, and regularizers
$\lambda_\ell$ defined above.
\STATE Initialize $I_\ell\gets\varnothing$ and
$\Sigma_{1,\ell}\gets\lambda_\ell I$ for every $\ell$.
\FOR{$k=1,\ldots,K$}
    \STATE After $x_k$ is known and before $y_k$ is observed:
    \STATE Set $w_k\gets0$ and
    $\Sigma_{k+1,\ell}\gets\Sigma_{k,\ell}$ for every $\ell$.
    \FOR{$\ell=1,\ldots,\ell_{\max}$}
        \IF{$\|x_k\|_{\Sigma_{k,\ell}^{-1}}>2^{-\ell}$}
            \STATE Set $w_k\gets
            2^{-\ell}/\|x_k\|_{\Sigma_{k,\ell}^{-1}}$.
            \STATE Assign the observation to this scale:
            $I_\ell\gets I_\ell\cup\{k\}$.
            \STATE Update only this matrix:
            $\Sigma_{k+1,\ell}\gets\Sigma_{k,\ell}+w_k^2x_kx_k^\top$.
            \STATE \textbf{break} the scale loop for this episode.
        \ENDIF
    \ENDFOR
\ENDFOR
\end{algorithmic}
\end{algorithm}

If no scale qualifies, the observation is assigned to no scale:
$w_k=0$ and all matrices remain unchanged. The PSRL episode continues
normally. The construction gives the two properties used below:
\begin{equation}
\Sigma_{k,\ell}
=\lambda_\ell I+\sum_{\substack{i<k\\i\in I_\ell}}w_i^2x_ix_i^\top,
\qquad
\|w_kx_k\|_{\Sigma_{k,\ell}^{-1}}=2^{-\ell}
\quad(k\in I_\ell).
\label{lin:eq:layer-update}
\end{equation}
The assignment and weight are chosen before $y_k$ is observed. Each
assigned observation increases the log determinant of its scale's matrix
by $\log(1+2^{-2\ell})$. Proposition~\ref{lin:prop:width-aggregation}
uses this identity to show $|I_\ell|2^{-2\ell}\le Cd\Lambda_K$.

\subsubsection{Confidence Tests and Coverage}
\label{lin:subsec:confidence-tests}
\label{lin:subsec:coverage}

The concentration bound depends on the true noise variances, but a
confidence set must be constructed from observed data alone. We use
squared prediction errors to set the tolerance: at the true parameter,
their weighted sum can be compared with the accumulated conditional
variance. Individual variances need not be estimated accurately.
Since the true parameter is unknown, we test each candidate using its own
prediction errors. Define
\begin{equation}
\begin{aligned}
Q_{k,\ell}(\vartheta)
&:=\sum_{\substack{i<k\\i\in I_\ell}}
w_i^2(y_i-\langle\vartheta,x_i\rangle)^2,\\
g_{k,\ell}(\vartheta)
&:=\lambda_\ell\vartheta+
\sum_{\substack{i<k\\i\in I_\ell}}
w_i^2x_i(\langle\vartheta,x_i\rangle-y_i).
\end{aligned}
\label{lin:eq:candidate-errors}
\end{equation}
Here $Q$ denotes the scalar weighted quadratic prediction loss, and $g$ is
the $d$-dimensional gradient of the regularized loss:
\[
g_{k,\ell}(\vartheta)
=\nabla_\vartheta\!\left[
\tfrac12 Q_{k,\ell}(\vartheta)
+\tfrac{\lambda_\ell}{2}\|\vartheta\|_2^2\right].
\]
We require the normalized size of
$g_{k,\ell}(\vartheta)$ to lie within a tolerance determined by
$Q_{k,\ell}(\vartheta)$. Before episode $k$, set
\begin{equation}
\cC_{k,h}:=
\left\{\vartheta:\|\vartheta\|_2\le B,\quad
\|g_{k,\ell}(\vartheta)\|_{\Sigma_{k,\ell}^{-1}}
\le32\,2^{-\ell}
\left[\sqrt{\Lambda_K Q_{k,\ell}(\vartheta)}+\Lambda_K\right]
\ \text{for all }\ell\le\ell_{\max}\right\}.
\label{lin:eq:confidence-set}
\end{equation}
This construction is used only in the analysis. The sets are closed subsets
of the bounded parameter ball and depend only on
earlier observations and posterior draws, so they are determined by $\cF_k$
without using $\theta_k$. In particular, the current feature $x_k$ and
weight $w_k$ are used only when incorporating the new observation into
future confidence sets.

\paragraph{Why the true parameter passes the tests.}
The next lemma controls the signed noise sum and compares
$Q_{k,\ell}(\theta_h)$ with the accumulated conditional variance.
The first two inequalities give coverage. This concentration step concerns
the true parameter; the effect of using another candidate in
$Q_{k,\ell}(\vartheta)$ is handled by the deterministic comparison that follows.

\Needspace{15\baselineskip}
\begin{lemma}[Simultaneous concentration and coverage]
\label{lin:lem:coverage}
Fix any admissible environment parameter $\theta$. With probability at least
$1-\delta$, the following bounds hold simultaneously for all $h\in[H]$,
$k\le K+1$, and retained scales $\ell$:
\begin{equation}
\begin{aligned}
\left\|\sum_{\substack{i<k\\i\in I_\ell}}
w_i^2x_i(y_i-\langle\theta_h,x_i\rangle)\right\|_{\Sigma_{k,\ell}^{-1}}
&\le16\,2^{-\ell}
\sqrt{\Lambda_K\sum_{\substack{i<k\\i\in I_\ell}}w_i^2\sigma_i^2}
+6\,2^{-\ell}\Lambda_K,\\
\sum_{\substack{i<k\\i\in I_\ell}}w_i^2\sigma_i^2
&\le2Q_{k,\ell}(\theta_h)+2\Lambda_K,\\
Q_{k,\ell}(\theta_h)
&\le2\sum_{\substack{i<k\\i\in I_\ell}}w_i^2\sigma_i^2+\Lambda_K.
\end{aligned}
\label{lin:eq:concentration-inputs}
\end{equation}

On the same event, $\theta_h\in\cC_{k,h}$ for every $h\in[H]$ and $k\le K$.
\end{lemma}

\begin{proof}
Fix a stage $h$ and a scale $\ell$, and write
$\varepsilon_i=y_i-\langle\theta_h,x_i\rangle$. Under the fixed environment,
\eqref{lin:eq:explicit-filtration} gives
\[
\E[\varepsilon_i\mid\mathcal H_{i,h}]=0,
\qquad
|\varepsilon_i|\le1,
\qquad
\E[\varepsilon_i^2\mid\mathcal H_{i,h}]=\sigma_i^2.
\]
Allocate $\delta_0:=\delta/(3H\ell_{\max})$ to each of the following three
bounds. For a sufficiently large universal constant in the definition of
$\Lambda_K$,
\[
\log(4K^2/\delta_0)\le\Lambda_K.
\]
All processes below use the filtration that reveals the scale assignment and
weight before the next-state value.

\paragraph{Vector concentration.}
Apply \citep[Theorem~2.1]{zhao2023} to the predictable vectors
$\mathbf1\{i\in I_\ell\}w_ix_i$ and mean-zero increments
$\mathbf1\{i\in I_\ell\}w_i\varepsilon_i$, with regularizer $\lambda_\ell$.
The regularized matrix is $\Sigma_{k,\ell}$, and both increments are zero
outside $I_\ell$. For $i\in I_\ell$,
\[
\|w_ix_i\|_{\Sigma_{i,\ell}^{-1}}=2^{-\ell},
\qquad
|w_i\varepsilon_i|\le1,
\qquad
\E[w_i^2\varepsilon_i^2\mid\mathcal H_{i,h}]=w_i^2\sigma_i^2.
\]
Thus, with probability at least $1-\delta_0$, simultaneously for $k\le K+1$,
\[
\left\|\sum_{\substack{i<k\\i\in I_\ell}}
w_i^2x_i\varepsilon_i\right\|_{\Sigma_{k,\ell}^{-1}}
\le16\,2^{-\ell}
\sqrt{\Lambda_K\sum_{\substack{i<k\\i\in I_\ell}}w_i^2\sigma_i^2}
+6\,2^{-\ell}\Lambda_K.
\]
This is the first inequality in \eqref{lin:eq:concentration-inputs}; the
empty sum at $k=1$ is deterministic.

\paragraph{Scalar comparison.}
For $i\in I_\ell$,
\[
0\le w_i^2\varepsilon_i^2\le1,
\qquad
\E[w_i^2\varepsilon_i^2\mid\mathcal H_{i,h}]=w_i^2\sigma_i^2.
\]
For $t\in[0,1]$, the inequalities $e^{-t}\le1-t/2$ and $e^t\le1+2t$ give
\[
\begin{aligned}
\E\!\left[
e^{w_i^2\sigma_i^2/2-w_i^2\varepsilon_i^2}
\,\middle|\,\mathcal H_{i,h}\right]
&\le e^{w_i^2\sigma_i^2/2}
\left(1-\frac{w_i^2\sigma_i^2}{2}\right)\le1,\\
\E\!\left[
e^{w_i^2\varepsilon_i^2-2w_i^2\sigma_i^2}
\,\middle|\,\mathcal H_{i,h}\right]
&\le e^{-2w_i^2\sigma_i^2}(1+2w_i^2\sigma_i^2)\le1.
\end{aligned}
\]
Use factor one outside $I_\ell$. Since scale membership is chosen before the
next-state value, the corresponding products are nonnegative supermartingales
starting at one. The maximal inequality, with failure probability $\delta_0$
for each product, gives simultaneously for all $k$
\[
\begin{aligned}
\sum_{\substack{i<k\\i\in I_\ell}}w_i^2\sigma_i^2
&\le2Q_{k,\ell}(\theta_h)+2\log(1/\delta_0),\\
Q_{k,\ell}(\theta_h)
&\le2\sum_{\substack{i<k\\i\in I_\ell}}w_i^2\sigma_i^2+\log(1/\delta_0).
\end{aligned}
\]
Here the squared-error sum is $Q_{k,\ell}(\theta_h)$ by
\eqref{lin:eq:candidate-errors}. Since $\log(1/\delta_0)\le\Lambda_K$,
these are the remaining two inequalities in
\eqref{lin:eq:concentration-inputs}. A union bound over the three events,
all stages, and all retained scales has total failure probability at most
$3H\ell_{\max}\delta_0=\delta$.

\paragraph{Coverage.}
On this simultaneous event, since
$\|\lambda_\ell\theta_h\|_{\Sigma_{k,\ell}^{-1}}
\le\sqrt{\lambda_\ell}B\le2^{-\ell}$, the first two bounds imply
\[
\|g_{k,\ell}(\theta_h)\|_{\Sigma_{k,\ell}^{-1}}
\le16\,2^{-\ell}
\sqrt{\Lambda_K\bigl(2Q_{k,\ell}(\theta_h)+2\Lambda_K\bigr)}
+7\,2^{-\ell}\Lambda_K
\le32\,2^{-\ell}
\left[\sqrt{\Lambda_K Q_{k,\ell}(\theta_h)}+\Lambda_K\right].
\]
Thus $\theta_h\in\cC_{k,h}$ for every $k,h$.

\end{proof}

\subsubsection{From Confidence Tests to Prediction Bounds}

Coverage alone does not ensure that the set is narrow. We now bound the
parameter-error factor in the Cauchy--Schwarz inequality from
Section~\ref{lin:subsec:construction}. A candidate's tolerance depends on
its own squared errors, so we first compare these errors with those of
the true parameter. At sufficiently fine scales, this comparison closes
and gives the following distance bound. The scale-assignment rule will
then control the remaining feature-uncertainty factor.

\begin{lemma}[Fine-scale candidate distance]
\label{lin:lem:candidate-comparison}
On the event in Lemma~\ref{lin:lem:coverage}, simultaneously for all stages
$h$, episodes $k\le K$, retained scales with
$2^\ell\ge128\sqrt{\Lambda_K}$, and $\vartheta\in\cC_{k,h}$,
\begin{equation}
\|\vartheta-\theta_h\|_{\Sigma_{k,\ell}}
\le C2^{-\ell}
\left[\sqrt{\Lambda_K
\sum_{\substack{i<k\\i\in I_\ell}}w_i^2\sigma_i^2}+\Lambda_K\right].
\label{lin:eq:candidate-distance}
\end{equation}
\end{lemma}

\begin{proof}
For $\vartheta\in\cC_{k,h}$, let
$D=\|\vartheta-\theta_h\|_{\Sigma_{k,\ell}}$. By linearity and the inequality
$(a+b)^2\le2a^2+2b^2$,
\begin{equation}
\begin{aligned}
g_{k,\ell}(\vartheta)-g_{k,\ell}(\theta_h)
&=\Sigma_{k,\ell}(\vartheta-\theta_h),\\
Q_{k,\ell}(\vartheta)
&\le2Q_{k,\ell}(\theta_h)
+2\sum_{\substack{i<k\\i\in I_\ell}}
w_i^2\langle\vartheta-\theta_h,x_i\rangle^2
\le2Q_{k,\ell}(\theta_h)+2D^2.
\end{aligned}
\label{lin:eq:candidate-comparison}
\end{equation}
The first identity turns the difference of test statistics into a
parameter difference. The second inequality bounds the candidate's own
squared errors by those of the true parameter and the same parameter
distance. Since both parameters pass the test in
\eqref{lin:eq:confidence-set},
\[
\begin{aligned}
D
&\le32\,2^{-\ell}
\left[\sqrt{\Lambda_K Q_{k,\ell}(\vartheta)}
+\sqrt{\Lambda_K Q_{k,\ell}(\theta_h)}+2\Lambda_K\right]\\
&\le32\,2^{-\ell}
\left[(\sqrt2+1)\sqrt{\Lambda_K Q_{k,\ell}(\theta_h)}
+\sqrt{2\Lambda_K}\,D+2\Lambda_K\right].
\end{aligned}
\]
For $2^\ell\ge128\sqrt{\Lambda_K}$, the coefficient of $D$ on the right is
less than $1/2$. Absorbing this term and applying the last bound in
\eqref{lin:eq:concentration-inputs} gives
\eqref{lin:eq:candidate-distance}.
\end{proof}

The comparison is deterministic once the true-parameter concentration event
holds; no concentration bound uniform over candidates is needed.

\subsubsection{Summing the Prediction Bounds}

We now combine the candidate-distance bound with the scale counts.
For a fine scale $\ell$, the preceding scale supplies two factors of
$2^{-(\ell-1)}$: one from feature uncertainty and one from the
candidate-distance bound. Their product offsets the $2^{2\ell}$ growth
in the bound on the number of assigned observations. The remaining
variance terms can then be summed across scales. At coarse scales we
use the clipped width bound of one; observations assigned to no scale
already have small uncertainty at the finest scale.

\begin{proposition}[Geometric-scale width aggregation]
\label{lin:prop:width-aggregation}
Fix a stage $h$ and use the construction in
Sections~\ref{lin:subsec:construction}--\ref{lin:subsec:confidence-tests}.
Suppose $\theta_h\in\cC_{k,h}$
for every $k\le K$ and \eqref{lin:eq:candidate-distance} holds for all
candidates and all retained scales with $2^\ell\ge128\sqrt{\Lambda_K}$.
Then, deterministically,
\begin{equation}
\sum_{k=1}^K
\left[1\wedge\sup_{\vartheta,\vartheta'\in\cC_{k,h}}
|\langle\vartheta-\vartheta',x_k\rangle|\right]
\le Cd\left[\Lambda_K^2\sqrt{\sum_{k=1}^K\sigma_k^2}+\Lambda_K^3\right].
\label{lin:eq:aggregated-width}
\end{equation}
\end{proposition}

\begin{proof}
\textbf{Number of observations per scale.}
The scale rule and the determinant identity give
\begin{equation}
|I_\ell|2^{-2\ell}
=\sum_{k\in I_\ell}\|w_kx_k\|_{\Sigma_{k,\ell}^{-1}}^2
\le2\log\frac{\det\Sigma_{K+1,\ell}}{\det\Sigma_{1,\ell}}
\le2d\log\!\left(1+\frac{K(1+B)^2\,2^{2\ell}}{d}\right)
\le Cd\Lambda_K.
\label{lin:eq:layer-count}
\end{equation}
Here $\sum_{k\in I_\ell}w_k^2\|x_k\|_2^2\le K$, and the definitions of
$\ell_{\max}$ and $\Lambda_K$ ensure $\ell_{\max}\le C\Lambda_K$ and the
last logarithmic bound.

\paragraph{Coarse scales.}
For $2^\ell<256\sqrt{\Lambda_K}$, use the clipped width
bound of one. Their total contribution is at most
\[
\sum_{\ell:\,2^\ell<256\sqrt{\Lambda_K}}|I_\ell|
\le Cd\Lambda_K
\sum_{\ell:\,2^\ell<256\sqrt{\Lambda_K}}2^{2\ell}
\le Cd\Lambda_K^2.
\]
\paragraph{Fine scales.}
For $2^\ell\ge256\sqrt{\Lambda_K}$, we use the preceding scale
$\ell-1$: it did not satisfy the assignment condition, since $\ell$ was
the first scale that did. Thus
$\|x_k\|_{\Sigma_{k,\ell-1}^{-1}}\le2^{-(\ell-1)}$ for $k\in I_\ell$.
The preceding scale satisfies $2^{\ell-1}\ge128\sqrt{\Lambda_K}$, so
\eqref{lin:eq:candidate-distance} applies there. Consequently,
\[
\begin{aligned}
\sup_{\vartheta,\vartheta'\in\cC_{k,h}}
|\langle\vartheta-\vartheta',x_k\rangle|
&\le C2^{-2(\ell-1)}
\left[\sqrt{\Lambda_K\sum_{i\in I_{\ell-1}}w_i^2\sigma_i^2}
+\Lambda_K\right],\\
\sum_{k\in I_\ell}
\left[1\wedge\sup_{\vartheta,\vartheta'\in\cC_{k,h}}
|\langle\vartheta-\vartheta',x_k\rangle|\right]
&\le Cd\Lambda_K
\left[\sqrt{\Lambda_K\sum_{i\in I_{\ell-1}}w_i^2\sigma_i^2}
+\Lambda_K\right].
\end{aligned}
\]
Here the factor $2^{-2(\ell-1)}$ in the individual width cancels
the $2^{2\ell}$ factor in the count bound, up to a constant.
It remains to sum the square roots of the weighted variances. The sets
$I_\ell$ are disjoint and $w_i\le1$, so Cauchy--Schwarz gives
\[
\sum_{\ell=1}^{\ell_{\max}}
\sqrt{\sum_{i\in I_\ell}w_i^2\sigma_i^2}
\le\sqrt{\ell_{\max}\sum_{k=1}^K\sigma_k^2},
\]
and all fine scales together contribute at most
$Cd[\Lambda_K^2\sqrt{\sum_k\sigma_k^2}+\Lambda_K^3]$.

\paragraph{Observations assigned to no scale.}
For these episodes,
$\|x_k\|_{\Sigma_{k,\ell_{\max}}^{-1}}\le2^{-\ell_{\max}}$.
Apply \eqref{lin:eq:candidate-distance} at this largest scale and use
$\sum_iw_i^2\sigma_i^2\le K$. Their total contribution is at most
\[
CK\,2^{-2\ell_{\max}}\bigl[\sqrt{\Lambda_K K}+\Lambda_K\bigr]
\le C\Lambda_K,
\]
because $2^{-\ell_{\max}}\le[4K(1+B)]^{-1}$. Adding the contributions
from the three cases proves \eqref{lin:eq:aggregated-width}.
\end{proof}

\paragraph{Completing the confidence-width guarantee.}
\begin{proof}[Proof of Lemma~\ref{lem:linear-sampled-width}]
The sets in \eqref{lin:eq:confidence-set} are closed subsets of the
parameter ball and are determined by $\cF_k$. Their construction is the
same for every admissible environment and uses no fresh posterior draw.
Lemma~\ref{lin:lem:coverage} supplies one event of probability at least
$1-\delta$ on which all true parameters are covered. On that event,
Lemma~\ref{lin:lem:candidate-comparison} gives the hypotheses of
Proposition~\ref{lin:prop:width-aggregation} for every stage. Since
$\Lambda_K\ge1$,
\[
\Lambda_K^2\sqrt{\sum_k\sigma_k^2}+\Lambda_K^3
\le\Lambda_K^4\left(\sqrt{\sum_k\sigma_k^2}+1\right).
\]
Restoring $x_{k,h}=x_k$ and $\sigma_{k,h}^2(\theta)=\sigma_k^2$ proves
\eqref{lin:eq:confidence-width} simultaneously for all stages.
\end{proof}

\section{Information-Theoretic Counterexample}
\label{app:framework-counterexample}

Information-theoretic Bayesian regret analyses typically involve two
prior-dependent quantities: how costly it is to obtain information about a
learning target, and how much uncertainty that target contains initially.
These are represented by the information-ratio constant $\Gamma_\rho^Z$ and
the target entropy $H_\rho(Z)$. The counterexample below shows that their
worst cases can occur under different priors even when their same-prior product
stays small. Thus replacing
$\sup_\rho[\Gamma_\rho^ZH_\rho(Z)]$ by the product of two separate suprema can
create an artificial polynomial loss. The standard information-ratio argument
itself remains valid.

\subsection{The standard information-ratio template}

A learning target is a random variable $Z=f(M^\star)$ chosen to represent what
the learner seeks to identify. Depending on the analysis, it may be the full
model when the model class is discrete, the optimal policy, or another
function of the unknown MDP. Let
$O_k=(S_{k,h},A_{k,h},R_{k,h},S_{k,h+1})_{h=1}^H$ denote the episode-$k$
observation and let
$\rho_k(\cdot):=\mathbb P(M^\star\in\cdot\mid\mathcal F_k)$. For a policy
$\pi$ executed under $\rho_k$, $I_{\rho_k}(Z;O_k\mid\pi)$ denotes the mutual information
between the target and the episode observation under the corresponding
posterior predictive law. For discrete $Z$,
\[
H_\rho(Z):=-\sum_z\mathbb P_\rho(Z=z)
\log\mathbb P_\rho(Z=z).
\]
Define $\Gamma_\rho^Z$ as the smallest constant such that, almost surely for
every $k\le K$,
\[
\bigl(\mathbb E_k[\Delta_k]\bigr)^2
\le \Gamma_\rho^Z I_{\rho_k}(Z;O_k\mid\pi_k).
\]
If no finite constant satisfies these inequalities, set
$\Gamma_\rho^Z=+\infty$.
Thus $I_{\rho_k}(Z;O_k\mid\pi)$ is the information obtained in one episode,
$\Gamma_\rho^Z$ measures how much posterior regret is paid per unit of
information, and $H_\rho(Z)$ measures the uncertainty present initially
\citep{russo2016,russo2018}. Related information-theoretic analyses for RL
appear in \citep{hao2022,moradipari2023}.

For the counterexample, consider a horizon-three MDP with
$\mathcal S=\{s_0,s_1,\ldots,s_A,g,b\}$ and $\mathcal A=[A]$.
In this example, the fixed initial state is $s_0$, so $\Delta_k$ and Bayesian
regret are evaluated at $s_0$ in place of $s_1$.
Stage one moves
from $s_0$ uniformly to one of the intermediate states
$s_1,\ldots,s_A$, independently of the chosen action. At stage two, each
$s_j$ has a unique optimal action, and we write
\[
a_j^\star:=\pi_{M^\star,2}^\star(s_j),
\qquad
a^\star=(a_1^\star,\ldots,a_A^\star).
\]
The stage-two transition law is
\[
P_2^{M^\star}(g\mid s_j,a)=\mathbf1\{a=a_j^\star\},
\qquad
P_2^{M^\star}(b\mid s_j,a)=\mathbf1\{a\ne a_j^\star\},
\]
All earlier rewards are zero, and the terminal rewards at $g$ and $b$ are one
and zero. Hence the target in
this counterexample is $a^\star$, exactly the stage-two optimal policy
restricted to the intermediate states, or equivalently the part of the
transition kernel that determines the optimal decisions.

\subsection{How information accumulates under PSRL}

The following identity is stated for a generic target $Z=f(M^\star)$; later we
apply it with $Z=a^\star$. The posterior draw determines which policy is
executed, but it is not itself an observation of the true MDP. The information
about $Z$ acquired in episode $k$ comes from the trajectory generated by that
sampled policy.

\begin{lemma}[Episode information under PSRL]
\label{lem:psrl-episode-information}
For every episode $k$,
\[
I(Z;\mathcal F_{k+1}\mid\mathcal F_k)
=\mathbb E\!\left[I_{\rho_k}(Z;O_k\mid\pi_k)\right].
\]
\end{lemma}

\begin{proof}
Since $Z$ is a function of $M^\star$ and $M_k$ is a fresh posterior draw,
\[
M_k\perp(M^\star,Z)\mid\mathcal F_k,
\qquad
I(Z;M_k\mid\mathcal F_k)=0.
\]
Moreover, $M_k$ affects the episode only through
$\pi_k=\pi_{M_k}^\star$. Once $\pi_k$ is fixed, the remaining details of
$M_k$ do not affect the observation because $O_k$ is generated by executing
$\pi_k$ in the true model $M^\star$. Thus
\[
\mathcal L(O_k\mid\mathcal F_k,M^\star,M_k)
=\mathcal L(O_k\mid\mathcal F_k,M^\star,\pi_k).
\]
The mutual-information chain rule now gives
\[
\begin{aligned}
I(Z;\mathcal F_{k+1}\mid\mathcal F_k)
&=I(Z;M_k,O_k\mid\mathcal F_k)
=I(Z;O_k\mid\mathcal F_k,M_k)\\
&=I(Z;O_k\mid\mathcal F_k,\pi_k)
=\mathbb E\!\left[I_{\rho_k}(Z;O_k\mid\pi_k)\right].
\end{aligned}
\]
The last equality averages the information gain of the realized PSRL policy
over the history and posterior draw.
\end{proof}

Combining Lemma~\ref{lem:psrl-episode-information} with the one-episode ratio
bound and Cauchy--Schwarz yields, for a generic discrete target $Z$,
\[
\operatorname{BReg}_\rho(K)
\le \mathbb E\sum_{k=1}^K
\sqrt{\Gamma_\rho^Z I_{\rho_k}(Z;O_k\mid\pi_k)}
\le \sqrt{K\Gamma_\rho^Z
\sum_{k=1}^K\mathbb E\!\left[
I_{\rho_k}(Z;O_k\mid\pi_k)\right]}
\le \sqrt{K\Gamma_\rho^Z H_\rho(Z)}.
\]
The last step uses the mutual-information chain rule: the total information
revealed across all episodes cannot exceed the initial entropy of the target.
The inequality is immediate when $\Gamma_\rho^Z=+\infty$.

\subsection{Two priors over the stage-two transition structure}

We now set the target to the optimal-action vector $a^\star$. The two priors
differ only in how these optimal actions, equivalently the relevant stage-two
transition rows, are related across the intermediate states.

Under the \emph{shared-optimal-action prior} $\rho_G$, one optimal action
$a_{\mathrm{sh}}^\star\sim\operatorname{Unif}([A])$ is shared by every
intermediate state:
\[
a_j^\star=a_{\mathrm{sh}}^\star,
\qquad j\in[A].
\]
Thus every $s_j$ has the same optimal action, and the relevant transition rows
are perfectly correlated through $a_{\mathrm{sh}}^\star$.

Under the \emph{state-wise independent prior} $\rho_L$,
\[
a_1^\star,\ldots,a_A^\star
\overset{\mathrm{iid}}{\sim}\operatorname{Unif}\{1,2\}.
\]
Thus each intermediate state has its own independently drawn optimal action;
learning the optimal action at one state gives no information about the
optimal action at another state.

Figure~\ref{fig:crossed-prior-mdp} records the common MDP structure.

\begin{figure}[ht]
\centering
\begin{tikzpicture}[
    >=stealth,
    state/.style={circle,draw,minimum size=2.25em,inner sep=1pt}
]
    \node[state] (s0) at (0,2.4) {$s_0$};
    \node[state] (sj) at (0,0.8) {$s_j$};
    \node[state] (g) at (-1.45,-1.0) {$g$};
    \node[state] (b) at (1.45,-1.0) {$b$};

    \draw[->] (s0) -- node[right] {$j\sim\operatorname{Unif}([A])$} (sj);
    \draw[->] (sj) -- node[above left] {$a=a_j^\star$} (g);
    \draw[->] (sj) -- node[above right] {$a\ne a_j^\star$} (b);

    \node at (-1.45,-1.55) {$r_3=1$};
    \node at (1.45,-1.55) {$r_3=0$};
\end{tikzpicture}
\caption{At stage one the MDP moves uniformly from $s_0$ to an intermediate
state $s_j$. At stage two, the unique optimal action $a_j^\star$ transitions
to $g$ and every other action transitions to $b$.}
\label{fig:crossed-prior-mdp}
\end{figure}
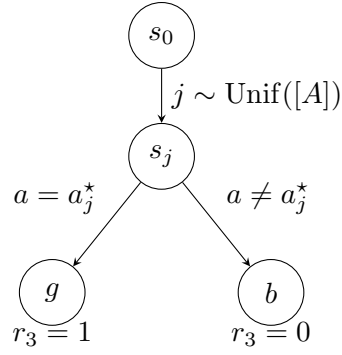

\begin{proposition}[Crossed worst cases]
\label{prop:separate-envelope-gap}
For every $A\ge4$ and $K\ge1$, the MDP and priors above satisfy
\[
\begin{aligned}
\Gamma_{\rho_G}^{a^\star}
&=\Theta\!\left(\frac{A}{\log A}\right),
& H_{\rho_G}(a^\star)&=\log A,\\
\Gamma_{\rho_L}^{a^\star}
&=\frac{1}{4\log2}=\Theta(1),
& H_{\rho_L}(a^\star)&=A\log2.
\end{aligned}
\]
Consequently,
\[
\begin{aligned}
\sup_{\rho\in\{\rho_G,\rho_L\}}
\bigl[\Gamma_\rho^{a^\star}H_\rho(a^\star)\bigr]
&=\Theta(A),\\
\left(\sup_{\rho\in\{\rho_G,\rho_L\}}
\Gamma_\rho^{a^\star}\right)
\left(\sup_{\rho\in\{\rho_G,\rho_L\}}
H_\rho(a^\star)\right)
&=\Theta\!\left(\frac{A^2}{\log A}\right).
\end{aligned}
\]
Since $S=A+3$, replacing the coupled supremum by the product of the two
separate suprema enlarges the resulting regret bound by
$\Theta(\sqrt{S/\log S})$.
\end{proposition}

\begin{proof}
\par\medskip\noindent\textbf{Shared-optimal-action prior $\rho_G$.}\ 
Only the common optimal action $a_{\mathrm{sh}}^\star$ is unknown, so
\[
H_{\rho_G}(a^\star)=\log A.
\]
Suppose the current posterior leaves $m\ge2$ possible values of
$a_{\mathrm{sh}}^\star$. Testing one remaining action succeeds with
probability $1/m$; a failure eliminates only that action. Therefore
\[
\begin{aligned}
\mathbb E_k[\Delta_k]
&=1-\frac1m,\\
I_{\rho_k}(a^\star;O_k\mid\pi_k)
&=-\frac1m\log\frac1m
-\left(1-\frac1m\right)\log\left(1-\frac1m\right)
=\Theta\!\left(\frac{\log m}{m}\right),\\
\frac{(\mathbb E_k[\Delta_k])^2}
{I_{\rho_k}(a^\star;O_k\mid\pi_k)}
&=\Theta\!\left(\frac m{\log m}\right).
\end{aligned}
\]
At a resolved history both the regret and information gain are zero. Since
$m/\log m\le C A/\log A$ for $2\le m\le A$, with equality in order at the
initial posterior,
\[
\Gamma_{\rho_G}^{a^\star}
=\Theta\!\left(\frac A{\log A}\right),
\qquad
\Gamma_{\rho_G}^{a^\star}H_{\rho_G}(a^\star)=\Theta(A).
\]
The uncertainty is small because one action determines the optimal policy at
all intermediate states, but identifying that shared action is expensive:
each failed episode removes at most one candidate.

\par\medskip\noindent\textbf{State-wise independent prior $\rho_L$.}\ 
Now the intermediate states carry independent binary optimal actions, so
\[
H_{\rho_L}(a^\star)=A\log2,
\qquad
\rho_k(a_j^\star=1)\in\left\{0,\frac12,1\right\}.
\]
The sum below counts the states whose optimal actions remain unresolved. Then
\[
\begin{aligned}
\mathbb E_k[\Delta_k]
&=\frac1{2A}\sum_{j=1}^A
\mathbf1\!\left\{\rho_k(a_j^\star=1)=\frac12\right\},\\
I_{\rho_k}(a^\star;O_k\mid\pi_k)
&=\frac{\log2}{A}\sum_{j=1}^A
\mathbf1\!\left\{\rho_k(a_j^\star=1)=\frac12\right\},\\
\frac{(\mathbb E_k[\Delta_k])^2}
{I_{\rho_k}(a^\star;O_k\mid\pi_k)}
&=\frac1{4A\log2}
\sum_{j=1}^A\mathbf1\!\left\{\rho_k(a_j^\star=1)=\frac12\right\}
\le\frac1{4\log2}.
\end{aligned}
\]
Equality holds at the initial posterior, hence
\[
\Gamma_{\rho_L}^{a^\star}=\frac1{4\log2}=\Theta(1),
\qquad
\Gamma_{\rho_L}^{a^\star}H_{\rho_L}(a^\star)=\frac A4=\Theta(A).
\]
The uncertainty is large because every intermediate state has its own unknown
optimal action, but information is cheap: the first informative visit to an
unresolved state identifies its binary optimal action.

\par\medskip\noindent\textbf{Comparison.}\ 
For either prior separately, the ratio--entropy product is only order $A$:
\[
\sup_{\rho\in\{\rho_G,\rho_L\}}
\Gamma_\rho^{a^\star}H_\rho(a^\star)=\Theta(A).
\]
However, the largest information ratio comes from the shared-optimal-action
prior $\rho_G$, whereas the largest entropy comes from the state-wise
independent prior $\rho_L$. Taking these two worst cases separately gives
\[
\left(\sup_{\rho\in\{\rho_G,\rho_L\}}\Gamma_\rho^{a^\star}\right)
\left(\sup_{\rho\in\{\rho_G,\rho_L\}}H_\rho(a^\star)\right)
=\Theta\!\left(\frac{A^2}{\log A}\right).
\]
The two suprema select different priors. Since $S=A+3$, this artificial
factorization enlarges the resulting regret bound by
$\Theta(\sqrt{S/\log S})$.
\end{proof}

\paragraph{Why the obstruction is not specific to an optimal-policy target.}
The notation above uses the optimal-action vector $a^\star$ because it makes
the MDP mechanism transparent, but the crossed-prior phenomenon is not tied
to this particular target. In the constructed model class, all rewards and
all transitions except the stage-two rows above are fixed. Consequently,
\[
a^\star
\quad\longleftrightarrow\quad
P_2^{M^\star}
\quad\longleftrightarrow\quad
M^\star
\]
is a one-to-one correspondence: $a^\star$ determines the unknown transition
kernel and hence the full model, while $a_j^\star$ can be recovered from the
unique action satisfying
$P_2^{M^\star}(g\mid s_j,a)=1$. With the fixed tie-breaking rule, the same is
true for the full optimal policy. Therefore, if the learning target is instead
$M^\star$, $P_2^{M^\star}$, or the full optimal policy, then for every prior in
this construction its entropy and its mutual information with $O_k$ are
exactly the same as those of $a^\star$. Hence the corresponding
information-ratio constant and the crossed-prior gap are unchanged.

More generally, the same argument applies to any target that is in one-to-one
correspondence with $a^\star$ on this model class. It does not claim the same
calculation for an arbitrary coarser target that discards some of the
optimal-action information.

The counterexample therefore does not invalidate the prior-dependent
information-ratio bound. It shows only that a sharp arbitrary-prior guarantee
cannot in general be obtained by maximizing the information ratio and target
entropy separately. An information-theoretic proof seeking a minimax-order
arbitrary-prior result must preserve their same-prior coupling, or avoid this
factorization altogether.

\section{Time-Homogeneous Specializations}
\label{app:homogeneous-specializations}

The preceding proofs treat the transition laws at different stages
separately. Under time homogeneity, observations can instead be pooled across
stages. The posterior-sampling identity and Bellman variance bounds are
unchanged; below we give only the changes to confidence sets and cumulative
widths. All pooled confidence sets still use only pre-episode data.

\subsection{Homogeneous tabular MDPs}

Assume $J_h^M(\cdot\mid s,a)=J_1^M(\cdot\mid s,a)$ for every $h$ and every
candidate $M$. The prior over these shared reward--transition laws may
correlate all state--action pairs. Set
$\beta_K^{\mathrm{hom}}=\log(c_0SAH(S+H)K^2)$.

\begin{corollary}[Homogeneous tabular guarantee]
\label{thm:homogeneous-tabular}
For every prior on homogeneous tabular models and every $K\ge1$, exact
vanilla PSRL satisfies
\begin{equation}
\operatorname{BReg}_{\rho}(K)
\le C\sqrt{SAH^2K\,\beta_K^{\mathrm{hom}}L_K}
+CS^2AH\,\beta_K^{\mathrm{hom}}L_K
+C\sqrt{SAHK\,\beta_K^{\mathrm{hom}}}
+CSAH^2+CH.
\label{eq:hom-tabular-bound}
\end{equation}
The bound may be replaced by its minimum with $HK$. In particular,
$\operatorname{BReg}_{\rho}(K)
=\widetilde O(H\sqrt{SAK}+SAH(S+H))$, whose leading term matches the
homogeneous tabular minimax rate up to logarithmic factors
\citep{azar2017,domingues2021}.
\end{corollary}

\begin{proof}
\textbf{Pooled confidence sets.}
For each $(s,a)$, use one reward--next-state observation stack, ordered by
visits across all stages. Apply the confidence tests of
Appendix~\ref{app:concentration-calibration} to every prefix available before
episode $k$, replacing $\beta_K$ by $\beta_K^{\mathrm{hom}}$.
The coordinate and reward tests use the shared laws $P_1^M,r_1^M$.
The candidate-aligned test is imposed for each of the $H$ directions
$V^\star_{t+1,M}$, $t\in[H]$: value functions remain stage dependent even
when the model is homogeneous. Let $\mathcal C_k$ be the set of models passing
these pooled tests. Under each fixed $M$, each pooled stack is
i.i.d., so the same concentration argument gives
\[
\mathbb P_M(M\notin\mathcal C_k\text{ for some }k\le K)
\le CSA(HK)(S+H)e^{-\beta_K^{\mathrm{hom}}}\le\frac1{4K}.
\]
Posterior calibration therefore gives
$\sum_k\mathbb P(\mathcal G_k^c)\le1/2$, where
$\mathcal G_k=\{M^\star,M_k\in\mathcal C_k\}$.
Lemma~\ref{lem:local-reference-bridge} now applies with the pooled count
\[
N^{\mathrm{hom}}_{k,h}
:=\sum_{\ell<k}\sum_{j=1}^H
\mathbf1\{(S_{\ell,j},A_{\ell,j})=(S_{k,h},A_{k,h})\}
\]
and $\beta_K^{\mathrm{hom}}$. The count table is frozen during each episode.

\textbf{Pooled counts and small-count visits.}
For one pair with pre-episode count $n\ge H$ and $q\le H$ visits in that
episode, $q\le n$ and hence
\[
\frac qn\le2\log\frac{n+q}{n},
\qquad
\frac q{\sqrt n}\le\sqrt2\sum_{j=n}^{n+q-1}\frac1{\sqrt j}.
\]
The first bound telescopes from a count at least $H$ to at most $KH$.
The second is summed over the sequential visits and then over the $SA$
pairs, using Cauchy--Schwarz as in Lemma~\ref{lem:visit-count-sums}.
Before a pair reaches count $H$ it contributes fewer than $H$ visits,
and the crossing episode contributes at most another $H$. Thus
\begin{equation}
\begin{gathered}
\sum_{k,h:\,N^{\mathrm{hom}}_{k,h}\ge H}\frac{1}{N^{\mathrm{hom}}_{k,h}}
\le CSA L_K,\qquad
\sum_{k,h:\,N^{\mathrm{hom}}_{k,h}\ge H}\frac{1}{\sqrt{N^{\mathrm{hom}}_{k,h}}}
\le C\sqrt{SAHK},\\
\sum_{k,h}\mathbf1\{N^{\mathrm{hom}}_{k,h}<H\}\le2SAH.
\end{gathered}
\label{eq:hom-pooled-counts}
\end{equation}
Since $|e^P_{k,h}|+|e^r_{k,h}|\le H+1$, the small-count contribution is
at most $CSAH^2$.

\textbf{Applying the existing variance bounds.}
Let $W_P$ and $W_r$ be the expected absolute transition and reward errors
summed only over $\mathcal G_k\cap\{N^{\mathrm{hom}}_{k,h}\ge H\}$.
The regret reduction and the pooled reward bound give
\[
\operatorname{BReg}_{\rho}(K)\le W_P+W_r+CSAH^2+CH,
\qquad W_r\le C\sqrt{SAHK\,\beta_K^{\mathrm{hom}}}.
\]
Proposition~\ref{prop:three-variance-bounds} is unchanged. Substituting
\eqref{eq:hom-pooled-counts} into the aggregation in
Appendix~\ref{app:regret-closure}, and including the small-count errors in
the variance bounds, gives
\[
\begin{aligned}
W_P\le{}&
C\sqrt{SAH^2K\,\beta_K^{\mathrm{hom}}L_K}
+CS^2AH\,\beta_K^{\mathrm{hom}}L_K\\
&+C\sqrt{S^2AH\,\beta_K^{\mathrm{hom}}L_K
\bigl(W_P+W_r+CSAH^2+\operatorname{BReg}_{\rho}(K)+H\bigr)}.
\end{aligned}
\]
Using the regret reduction inside the square root and applying Young's
inequality, exactly as in Appendix~\ref{app:regret-closure}, yields
\eqref{eq:hom-tabular-bound}.
\end{proof}

\subsection{Homogeneous linear-mixture MDPs}

Specialize Appendix~\ref{app:linear-mixture} to
$r_h=r_1$, $\phi_h=\phi_1$, and $\theta_h=\theta_1$ for every $h$.
Rewards remain known and deterministic, and the prior on the shared
transition parameter is arbitrary.

\begin{corollary}[Homogeneous linear-mixture guarantee]
\label{thm:homogeneous-linear-mixture}
For every prior on the admissible homogeneous parameter tuples and every
$K\ge1$, exact vanilla PSRL satisfies
\begin{equation}
\operatorname{BReg}_{\rho}(K)
\le Cd\sqrt{H^2K}\,\Lambda_K^4
+CdH^2\Lambda_K^4+Cd^2H\Lambda_K^8+CH.
\label{eq:hom-linear-bound}
\end{equation}
The bound may be replaced by its minimum with $HK$. In particular,
$\operatorname{BReg}_{\rho}(K)
=\widetilde O(d\sqrt{H^2K}+dH^2+d^2H)$.
\end{corollary}

\begin{proof}
\textbf{One confidence set for the shared parameter.}
Use a single copy of the construction in
Sections~\ref{lin:subsec:construction}--\ref{lin:subsec:confidence-tests},
with budget $HK$ and failure probability $\delta=(4HK)^{-2}$.
Index the observations chronologically by $t=(k-1)H+h$ and set
\[
x_t=\frac{\phi_{h,V^\star_{h+1,\theta_k}}(S_{k,h},A_{k,h})}{H},
\qquad
y_t=\frac{V^\star_{h+1,\theta_k}(S_{k,h+1})}{H},
\qquad
\sigma_t^2=\frac{\operatorname{Var}_{P_h^\theta}
(V^\star_{h+1,\theta_k})}{H^2}.
\]
As in \eqref{lin:eq:sampled-value-data}, under fixed $\theta$ the direction
$x_t$ is determined before $y_t$, and $y_t$ has conditional mean
$\langle\theta_1,x_t\rangle$ and conditional variance $\sigma_t^2$.
Use the same $\lambda_\ell$ and
$\ell_{\max}=\lceil\log_2\max\{4HK(1+B),256\sqrt{\Lambda_K}\}\rceil$.
The larger observation budget is absorbed by the universal constant in
$\Lambda_K$, and $\ell_{\max}\le C\Lambda_K$.

For scale selection and weights in Algorithm~\ref{alg:analysis-scale-assignment},
use the pre-episode matrix $\Sigma_{(k-1)H+1,\ell}$ throughout episode $k$.
Accumulate the weighted observations chronologically in $\Sigma_{t,\ell}$,
$Q_{t,\ell}$, and $g_{t,\ell}$, and define $\mathcal C_k$ by
\eqref{lin:eq:confidence-set} at time $(k-1)H+1$, with the common parameter
in place of $\theta_h$. Thus $\mathcal C_k$ uses only $\mathcal F_k$.
Because $\Sigma_{t,\ell}\succeq\Sigma_{(k-1)H+1,\ell}$ within an episode,
every observation assigned to $\ell$ satisfies
\[
\|w_tx_t\|_{\Sigma_{t,\ell}^{-1}}
\le\|w_tx_t\|_{\Sigma_{(k-1)H+1,\ell}^{-1}}=2^{-\ell}.
\]
Consequently, the proofs of Lemmas~\ref{lin:lem:coverage}
and~\ref{lin:lem:candidate-comparison} apply at every episode boundary,
with $HK$ observations and failure allocation $\delta/(3\ell_{\max})$.
In particular, $\theta_1\in\mathcal C_k$ for all $k$ with probability
at least $1-\delta$.

\textbf{The cost of keeping the matrices fixed within an episode.}
If a scale's determinant grows by at most a factor of two during episode
$k$, positive definiteness gives
\[
\det\Sigma_{kH+1,\ell}\le2\det\Sigma_{(k-1)H+1,\ell}
\quad\Longrightarrow\quad
\Sigma_{kH+1,\ell}\preceq2\Sigma_{(k-1)H+1,\ell}.
\]
Indeed, all eigenvalues of
$\Sigma_{(k-1)H+1,\ell}^{-1/2}\Sigma_{kH+1,\ell}
\Sigma_{(k-1)H+1,\ell}^{-1/2}$ are at least one and their product is at
most two, so each is at most two. Each observation assigned to $\ell$
in such an episode therefore has chronological squared inverse norm
at least $2^{-2\ell-1}$. The determinant argument in
Proposition~\ref{lin:prop:width-aggregation} then gives
\[
\#\{t=(k-1)H+h\in I_\ell:
\det\Sigma_{kH+1,\ell}\le2\det\Sigma_{(k-1)H+1,\ell}\}\,2^{-2\ell}
\le C\log\frac{\det\Sigma_{HK+1,\ell}}{\det\Sigma_{1,\ell}}
\le Cd\Lambda_K.
\]
This is the same count bound used in that proposition. Its coarse- and
fine-scale arguments apply to these observations: minimal scale selection
still bounds the current direction in the preceding scale's frozen matrix,
and the candidate comparison holds in that same matrix. Observations
assigned to no scale contribute at most
$CHK\,2^{-2\ell_{\max}}(\sqrt{\Lambda_K HK}+\Lambda_K)\le C\Lambda_K$.

For each scale, the total log-determinant increase is at most $Cd\Lambda_K$,
so at most $Cd\Lambda_K$ episodes can more than double its determinant.
Across all scales there are at most $Cd\Lambda_K^2$ such episodes.
Charging their at most $H$ observations the clipped width bound of one
adds at most $CdH\Lambda_K^2$. We have therefore proved, on the same
coverage event,
\begin{equation}
\sum_{t=1}^{HK}
\left[1\wedge\sup_{\vartheta,\vartheta'\in\mathcal C_{\lceil t/H\rceil}}
|\langle\vartheta-\vartheta',x_t\rangle|\right]
\le Cd\Lambda_K^4\left(\sqrt{\sum_{t=1}^{HK}\sigma_t^2}+H\right).
\label{eq:hom-linear-width}
\end{equation}

\textbf{Applying posterior calibration and variance closure.}
Write $W_P=\sum_{k,h}\mathbb E|e^P_{k,h}|$.
The sets $\mathcal C_k$ use only pre-episode data, so the posterior
calibration in Appendix~\ref{app:linear-mixture} applies unchanged.
Since $|e^P_{k,h}|\le H$, the failure contribution is at most
$2H^2K\delta\le C$. Multiplying \eqref{eq:hom-linear-width} by $H$ and
applying Jensen's inequality gives
\[
W_P\le Cd\Lambda_K^4
\sqrt{\sum_{k,h}\mathbb E\operatorname{Var}_{P_h^{\theta^\star}}
(V^\star_{h+1,\theta_k})}
+CdH^2\Lambda_K^4+CH.
\]
Proposition~\ref{prop:linear-cross-model-variance} bounds the variance
sum by $2H^2K+2HW_P$. Thus
\[
W_P\le CdH\sqrt K\,\Lambda_K^4
+Cd\Lambda_K^4\sqrt{HW_P}+CdH^2\Lambda_K^4+CH.
\]
Young's inequality absorbs the square-root term into
$\tfrac12W_P+Cd^2H\Lambda_K^8$. Together with
$\operatorname{BReg}_{\rho}(K)\le W_P$, this proves
\eqref{eq:hom-linear-bound}.
\end{proof}

For $B>1$ and $K\ge\max\{3d^2,(d-1)/(192(B-1))\}$,
the homogeneous construction of \citep{zhougu2022} has regret
$\Omega(d\sqrt K)$ when total episode reward is at most one. Its only positive
reward is $1/H$; scaling it by $H$ gives $\Omega(d\sqrt{H^2K})$ under the
present per-stage $[0,1]$ normalization. Thus the leading term in
Corollary~\ref{thm:homogeneous-linear-mixture} matches this lower bound up to
logarithmic factors.

\Needspace{32\baselineskip}
\section{Model-Independent Initial-State Distributions}
\label{app:initial-state}

The main text and the preceding proofs use a fixed, known initial state
$s_1$ in every episode. Consider instead an unknown fixed distribution $\mu$
on $\mathcal S$: each $S_{k,1}\sim\mu$ is drawn independently of the model,
the preceding history, and the current PSRL randomization, and is observed
before the first action. Planning still produces a policy optimal at every
state, so knowledge of $\mu$ is unnecessary.

In the regret decomposition, replace $s_1$ by the observed $S_{k,1}$ and set
$\Delta_k=V^\star_{1,M^\star}(S_{k,1})-V^{\pi_k}_{1,M^\star}(S_{k,1})$,
so Bayesian regret is $\mathbb E\sum_k\Delta_k$.
Conditioning additionally on this state leaves both model laws unchanged:
\[
\mathbb P(M^\star\in\cdot\mid\mathcal F_k,S_{k,1})
=\mathbb P(M_k\in\cdot\mid\mathcal F_k,S_{k,1})
=\rho_k(\cdot).
\]
Thus Lemma~\ref{lem:posterior-simulation-identities} applies with $S_{k,1}$
in place of $s_1$: posterior sampling cancels the optimal-value difference,
and the simulation identity yields the same local errors.

In Lemma~\ref{lem:bellman-square-telescope}, the initial term becomes
$-v_1(S_{k,1})^2$. After averaging over the initial state, it is
\[
-\mathbb E\!\left[v_1(S_{k,1})^2
\mid\mathcal F_k,M^\star,M_k\right]\le0,
\]
so it can still be dropped. The variance bounds in
Propositions~\ref{prop:three-variance-bounds}
and~\ref{prop:linear-cross-model-variance} therefore remain valid after
averaging over $S_{k,1}$. The tabular confidence tests use the same row
observations, and the visit-count bounds depend only on the numbers of rows
and visits. In the linear-mixture proof, each sampled value direction is
still determined before its next-state observation, so the concentration
argument also applies unchanged. Pooling observations across stages in the
homogeneous cases requires no common initial state either.

Taking expectations over each initial state and summing over episodes
therefore yields the same Bayesian regret guarantees as in
Theorems~\ref{thm:arbitrary-correlated-prior} and~\ref{thm:linear-mixture},
and the homogeneous specializations in Appendix~\ref{app:homogeneous-specializations},
without an additional term for estimating $\mu$. The guarantees hold
uniformly over all such model-independent distributions, known or unknown.
If the initial-state law depends on $M^\star$, the observed $S_{k,1}$
must instead be incorporated into the posterior before drawing $M_k$.

\end{document}